\documentclass[journal]{IEEEtran}
\ifCLASSINFOpdf
\else
\fi
\usepackage{subfigure}
\usepackage{amsmath}
\usepackage{graphicx}
\usepackage{color}
\usepackage{amssymb}
\usepackage{amsthm}
\usepackage{multirow}
\usepackage{makecell}
\usepackage{algorithm}
\usepackage{algorithmic}
\usepackage{makecell}
\usepackage{colortbl}
\usepackage{adjustbox}
\usepackage{booktabs}
\usepackage[colorlinks,linkcolor=blue]{hyperref}
\definecolor{Gray}{gray}{0.9}

\begin{document}
%
\title{Exploiting Non-Local Priors via Self-Convolution For Highly-Efficient Image Restoration}
%
%
%

\author{Lanqing~Guo, Zhiyuan~Zha,~\IEEEmembership{Member,~IEEE}, Saiprasad~Ravishankar,~\IEEEmembership{Senior Member,~IEEE,}
        and Bihan~Wen,~\IEEEmembership{Member,~IEEE}
\IEEEcompsocitemizethanks{
\IEEEcompsocthanksitem L. Guo, Z. Zha and B. Wen are with School of Electrical \& Electronic Engineering, Nanyang Technological University, Singapore 639798.  E-mail: (lanqing001@e.ntu.edu.sg, zhiyuan.zha@ntu.edu.sg, bihan.wen@ntu.edu.sg).
\IEEEcompsocthanksitem S. Ravishankar is with the Departments of Computational Mathematics, Science and Engineering, and Biomedical Engineering at Michigan State University, East Lansing, MI, USA 48824. E-mail: ravisha3@msu.edu.
}
}

%
%

\markboth{Journal of \LaTeX\ Class Files,~2021}%
{Shell \MakeLowercase{\textit{et al.}}: Bare Demo of IEEEtran.cls for IEEE Journals}
%


\newcommand{\add}[1]{\textcolor{black}{#1}}

\newcommand{\zhiyuan}[1]{\textcolor{red}{#1}}

\newcommand{\argmax}{\operatorname{arg\,max}}
\newcommand{\argmin}{\operatorname{arg\,min}}
\newcommand{\ie}{\textit{i.e., }}
\newcommand{\eg}{\textit{e.g., }}
\newcommand{\st}{\textit{s.t.}}
\def\etal{\emph{et al.}}
\newcommand{\Cmb}{\mathbb{C}}
\newcommand{\Rmb}{\mathbb{R}}
\newtheorem{theorem}{Theorem}
\newtheorem{lemma}{Lemma}
\newtheorem{corollary}{Corollary}
\newtheorem{conj}{Conjecture}
\newtheorem{proposition}{Proposition}
\newtheorem{definition}{Definition}
\newtheorem{claim}{Claim}
\newtheorem{remark}{Remark}

\maketitle

\begin{abstract}
Constructing effective image priors is critical to solving ill-posed inverse problems in image processing and imaging. 
Recent works proposed to exploit image non-local similarity for inverse problems by grouping similar patches and demonstrated state-of-the-art results in many applications.
However, compared to classic methods based on filtering or sparsity, most of the non-local algorithms are time-consuming, mainly due to the highly inefficient and redundant block matching step, where the distance between each pair of overlapping patches needs to be computed.
In this work, we propose a novel Self-Convolution operator to exploit image non-local similarity in a self-supervised way.
The proposed Self-Convolution can generalize the commonly-used block matching step and produce equivalent results with much cheaper computation.
Furthermore, by applying Self-Convolution, we propose an effective multi-modality image restoration scheme, which is much more efficient than conventional block matching for non-local modeling.
Experimental results demonstrate that (1) Self-Convolution can significantly speed up most of the popular non-local image restoration algorithms, with two-fold to nine-fold faster block matching, and (2) the proposed multi-modality image restoration scheme achieves superior denoising results in both efficiency and effectiveness on RGB-NIR images. 
The code is publicly available at \href{https://github.com/GuoLanqing/Self-Convolution}{https://github.com/GuoLanqing/Self-Convolution}.
\end{abstract}


%
\IEEEpeerreviewmaketitle

\section{Introduction}
\IEEEPARstart{I}{mage} restoration (IR) refers to the process of forming an
image from a collection of measurements.
Despite today’s vast
improvement in camera sensors, digital images are often still
corrupted by severe noise in low-light conditions, which might degrade the subsequent vision tasks.
Thus,
image restoration or reconstruction is a challenging and essential inverse
problem, where we aim to recover the underlying clean image $\mathbf{x} \in \Rmb^{N}$ from its degraded/noisy measurements $\mathbf{y} \in \Rmb^{M}$, which under a linear imaging model satisfy $\mathbf{y} = \mathbf{A} \mathbf{x} + \mathbf{e}$. 
Here, $\mathbf{A}$ and $\mathbf{e}$ denote the sensing operator and additive noise, respectively.
Different forms of $\mathbf{A}$ and $\mathbf{e}$ correspond to a range of important setups, including image denoising, inpainting, super-resolution, compressed sensing, etc. 
As most IR problems are ill-posed, exploiting effective image priors as regularizers is critical to the success of many IR schemes.

Classic IR algorithms take advantage of local image properties. One popular model is \textit{sparsity}, \ie patches of natural images are known to be sparse in some transform or dictionary domains~\cite{elad2006image, ravishankar2012learning, mairal2009non, xu2018trilateral}, thus one can reduce noise or artifacts by exploiting the low-dimensionality of the image patches compared to the noise. 
Beyond local properties, images contain non-local structures, \eg non-local self-similarity (NSS) in the spatial domain \cite{buades2005non}.
Many recent IR algorithms~\cite{dabov2006image, gu2014weighted,wen2017strollr2d,dong2012nonlocally,xu2015patch, guo2017patch,zhang2014group} exploit NSS by grouping similar patches using block matching (BM), followed by applying a regularizer (\eg based on low-rankness, joint sparsity, etc.) to each group, which have led to state-of-the-art IR results.
BM originated from the field of video processing and was very popular among several motion estimation techniques~\cite{kilthau2002full, koga1981motion}. Dabov \etal~\cite{dabov2006image} first applied the BM concept in image processing to exploit the NSS property of images.

However, the BM operator used by the non-local algorithms requires computing the euclidean distance between every pair of patches, which is computationally very expensive.
There are many existing Fast Block Matching methods~\cite{kilthau2002full, koga1981motion} to address the high computational cost of BM in motion estimation. Most of them examine only a subset of the possible locations within a search window, which can be used for computations much more efficiently than full search.
Different from BM in motion estimation, almost all non-local IR algorithms work with overlapping image patches in BM to suppress blocky artifacts, which makes the computation highly redundant.
These drawbacks limit the efficiency of non-local IR algorithms, especially when processing large-scale and high-dimensional data (\eg multi-channel images, volumetric data, and videos)~\cite{xu2017multi,wen2017joint}.

To address this problem, we propose a novel Self-Convolution operator for non-local image modeling. It can reduce redundancy in non-local similarity calculations to a certain extent, and
exploit NSS in a highly efficient way. We prove that the Self-Convolution formulation can unify several non-local modeling methods, and can be made exactly equivalent to the BM operation.
To our best knowledge, Self-Convolution is the first work to investigate BM acceleration in non-local traditional image restoration methods. 
Since BM is a decoupled step from the subsequent restoration regularizer, Self-Convolution can be flexibly plugged into any BM based image restoration algorithm.
With such merits of Self-Convolution, we further propose an effective denoising scheme for multi-modality (MM) images, called online Self-MM. 
It applies the effective STROLLR model~\cite{wen2017strollr2d, wen2020image} for denoising using an online learning framework that is more scalable than batch learning, and achieves superior denoising results in both efficiency and effectiveness.

Our contributions are summarized as follows:
\begin{itemize}
\item We propose a novel Self-Convolution operator to exploit image NSS in a self-supervised and highly efficient way, which generalizes several commonly-used metrics. 
\item We prove that Self-Convolution can be made exactly equivalent to the commonly used BM operations. Self-Convolution can significantly speed up many popular non-local IR algorithms with $2\times$ to $9\times$ faster BM.
\item We propose a multi-modality image denoising scheme using Self-Convolution and online learning, called online Self-MM. The proposed online Self-MM algorithm outperforms the popular or state-of-the-art algorithms for denoising RGB-NIR images.
\end{itemize}

This paper is an extension of our very recent conference work~\cite{9414124} that briefly investigated a specific Self-Convolution operator.
Compared with this earlier work, here we provide detailed proofs for theoretical support.
Moreover, we propose a novel multi-modality (MM) image denoising scheme that effectively exploits the proposed Self-Convolution. Our experimental results illustrate the properties of the proposed methods and their performance on several datasets, with extensive evaluation and comparisons to prior or related methods.

The remainder of this paper is organized as follows. Section II reviews related works and relevant application domains.
Section III presents the formulation of our Self-Convolution algorithm and its generalizations.
Section IV introduces the proposed online Self-MM model and the corresponding learning formulation. 
Section V demonstrates the efficiency of the proposed Self-Convolution for accelerating various image restoration algorithms and evaluates the performance of the proposed online Self-MM over multi-modality dataset. Section VI concludes with proposals for future work.

\section{Related Work}
In this section, we review existing block matching based image restoration methods and Fast Fourier Transform accelerated convolution in the literature, as well as the difference between our Self-Convolution and self-attention mechanism or non-local neural networks.
\subsection{Block matching based image restoration} 
Employing similar data patches from different locations is widely-used in video processing under the term ``block matching (BM)", where it is used to improve the compression rate by exploiting similarity
among blocks that follow the motion of objects in consecutive frames. 
Traditionally, BM has found successful application in many video processing techniques such as video compression (MPEG standards)~\cite{lin1997fast} and object tracking~\cite{gyaourova2003block}.
The well-known BM3D~\cite{dabov2007image} algorithm first applied the BM concept to the image processing field to group similar patches.
Since then, methods exploiting BM-based non-local similarity have become more prevalent in the image processing field. Table~\ref{tab:BM restoration methods} summarizes some typical algorithms for different block matching based image restoration applications.

\vspace{1mm}
\noindent\textbf{Denoising}.
Image denoising is one of the most important problems in image processing and low-level computer vision.
The goal is to recover high-quality images from their corrupted
measurements, which also improves robustness in various
high-level vision tasks.
Natural images contain non-local structures, such as non-local self-similarity.
Gu \etal~\cite{gu2014weighted} incorporated low-rank matrix approximations using the weighted nuclear norm, where BM is used to construct low-rank matrices that enable denoising. Wen \etal~\cite{wen2020image} proposed a denoising method that jointly exploits the sparsity and low-rankness of data matrices formed by BM.
Besides, BM-based methods can be generalized to video processing. VBM4D~\cite{maggioni2012video} extended BM3D with motion estimation to track objects as they moved throughout the scene. 
Dong \etal~\cite{dong2015low} proposed a multi-frame image denoising algorithm that uses BM to extract similar 3D patches. 
Wen \etal~\cite{wen2017joint} extended the original spatial BM denoising algorithm in~\cite{wen2017strollr2d} to spatio-temporal data to search similar groups across frames. 
He \etal~\cite{he2019non} proposed a hyperspectral image denoising method by grouping similar cubes via BM, and imposing spatial similarity and spectral low-rank structural priors on these groups.

\begin{table}[!t]
\centering
\setlength{\tabcolsep}{0.5em}
\adjustbox{width=1.\linewidth}{
    \begin{tabular}{c|c cc}
        \toprule
         \textbf{Method} & Denoising & Super-Resolution & Inpainting\\ \midrule
       \multirow{5}{*}{\makecell[c]{Gray-Scale \\ Image}} & BM3D~\cite{dabov2006image} & LRTV~\cite{shi2015lrtv} & GSR~\cite{zhang2014group}   \\ 
       & WNNM~\cite{gu2014weighted} & LANR-NLM~\cite{jiang2016single} & NGS~\cite{liu2016nonlocal} \\
       & STROLLR~\cite{wen2017strollr2d} & NLRT~\cite{feng2016compressive} & TSLRA~\cite{guo2017patch} \\
       & GHP~\cite{zuo2013texture} &  JRSR~\cite{chang2018single}& MARLow~\cite{li2016marlow}  \\
    \midrule
       \multirow{3}{*}{\makecell[c]{HD \\ Image}} & NGMeet~\cite{he2019non} & LTTR~\cite{dian2019learning} & NGMeet~\cite{he2019non} \\ 
       & SALT~\cite{wen2017joint}& WLRTR~\cite{chang2020weighted} &NL-LRTC~\cite{ji2018nonlocal} \\
        &NLR-CPTD~\cite{xue2019nonlocal} & NLSTF~\cite{dian2017hyperspectral} & PM-MTGSR~\cite{li2016patch}  \\
        \bottomrule
    \end{tabular}
}
\vspace{0.1in}
    \caption{Some block matching based image restoration algorithms, including different image restoration applications for gray-scale image and high-dimensional (HD) images.}
\label{tab:BM restoration methods}
\end{table}

\vspace{1mm}
\noindent\textbf{Super-resolution}.
Image super-resolution aims to estimate a high-resolution (HR) image from a single or a series of relevant low-resolution (LR) images.
The local property based methods tend to smear out image details, while sometimes they have difficulty with recovering fine structures and textures.
The non-local self-similarity prior takes advantage of the redundancy of similar
patches in natural images to predict the surrounding pixels.
Shi~\etal~\cite{shi2015lrtv} first integrated both local and global information for effective image super-resolution under the constraints of total variation (TV) and low-rankness.
Following that, various methods~\cite{jiang2016single,feng2016compressive,chang2018single} incorporated non-local self-similarity
and local geometry priors into the learning from LR space to HR space.
With the development of devices recently, 3D data, such as Hyperspectral Images and Videos, has gained more and more usage due to the richer spectral or temporal information than single image.
Hence, many methods~\cite{dian2017hyperspectral,dian2019learning,chang2020weighted} extended the spatially non-local similarity to spectral or temporal domains.

\vspace{1mm}
\noindent\textbf{Inpainting}.
Image inpainting refers to the process of recovering the missing pixels in an image.
The inpainting problem includes editing (e.g., object removal), texture synthesis, content-aware image resizing (e.g., image enlargement), etc. 
In this paper, we focus on inpainting for image recovery, in which the missing regions are generally
small (e.g., randomly missing pixels), and the goal of inpainting
is to estimate the underlying complete image.
Similar to denoising, successful image inpainting algorithms~\cite{li2011image,zhang2014group,li2016marlow,guo2017patch}
have exploited sparsity or non-local image structures. 
For instance,
Li~\etal~\cite{li2011image} proposed to iteratively group similar patches and reconstruct each group via sparse approximation.
Zhang~\etal~\cite{zhang2014group} applied dictionary learning
within each group of similar patches, for improved sparse
representation in inpainting problems.
Li~\etal~\cite{li2016marlow} exploited the cross-dimensional correlation in similar patches collected from a single image.


\subsection{Fast Fourier Transform- (FFT-) accelerated convolution}
The well-known and widely-used Fast Fourier Transform (FFT), proposed by Cooley and Tukey~\cite{cooley1965algorithm}, has made a significant impact in the signal processing field.
The convolution theorem states that the convolution of any two signals $\mathbf{f}$ and $\mathbf{g}$ can be performed using Fourier transforms via
\begin{equation}
\mathbf{f} \underset{}{*}\mathbf{g}=\mathbf{F}^H(\mathbf{F}(\mathbf{f}) \cdot \mathbf{F}(\mathbf{g}))
\end{equation}
\add{Here, $\underset{}{*}$, $\mathbf{F}$ and $\mathbf{F}^{H}$ denote the operators of convolution, Fourier and inverse Fourier transforms, respectively.}
Much work in deep learning has been devoted to enhancing the efficiency of and accelerating convolutional layers. Mathieu \etal~\cite{mathieu2013fast} explored the possibility of using FFTs to reduce the arithmetic complexity of convolutions, which was followed by the work of Vasilache \etal~\cite{vasilache2014fast}.

\subsection{Self-attention mechanism and non-local networks} 

\noindent\textbf{Self-attention.}
The self-attention mechanism proposed by Vaswani \etal~\cite{vaswani2017attention} exploits dependencies at each position based on the global context.
Self-attention modules originated and have been widely used in machine translation and natural language processing~\cite{vaswani2017attention,dai2019transformer}. 
This has inspired applications of self-attention and related ideas in image synthesis~\cite{zhang2019self}, image captioning~\cite{yang2016stacked,chen2017sca}, and video prediction~\cite{jia2016dynamic,wang2018non}.
Applications of self-attention in computer vision were complementary to convolution: forms of
self-attention were primarily used to create layers that were
used in addition to, to modulate the output of, or otherwise
in combination with convolutions~\cite{zhao2020exploring}.
For channel-wise attention models~\cite{wang2017residual,hu2018squeeze}, attention weights reweight different channels.
Some approaches~\cite{chen2017sca,woo2018cbam,fu2019dual} incorporated spatial and channel-wise attention to exploit interdependencies across different dimensions.
A number of methods learned to reweight convolutional activations or offset the taps of convolutional kernels~\cite{zhu2019deformable} to reconstruct convolutional features.
Others applied self-attention in specific modules, then fused it into the whole convolutional structure~\cite{zhao2018psanet,wang2018non}.

\vspace{1mm}
\noindent\textbf{Non-local networks.}
The non-local network proposed by Wang \etal~\cite{wang2018non} is a generalization of the classical non-local means operation~\cite{buades2005non} in computer vision, and involves computing the response at a position as a weighted sum of the features at all positions in the input feature maps.
Many researchers noticed the power of the long range dependencies in different image tasks, \eg semantic segmentation~\cite{zhu2019asymmetric,liu2015parsenet,zhao2017pyramid}, object detection~\cite{luo2017non,yang2020context}, as well as image restoration~\cite{liu2018non,plotz2018neural}.

While the self-attention mechanism and non-local networks both exploit feature-domain non-local correlation, the proposed Self-Convolution focuses on the spatial self-similarity of natural images exploited in many frameworks.

\section{Self-Convolution} \label{sec:selfconv}
We start by reviewing several non-local methods based on existing works, and then propose a non-local modeling framework to unify these methods using Self-Convolution that leads to highly-efficient computation. 

To facilitate the discussion, we define the vectorization operation on matrices and tensors (unfolding along the first mode), as $\mathrm{vec}(\cdot): \mathbb{R}^{\sqrt{n} \times \sqrt{n}} \rightarrow \mathbb{R}^n$ and $\mathrm{vecT}(\cdot): \mathbb{R}^{\sqrt{n} \times \sqrt{n} \times p} \rightarrow \mathbb{R}^{np}$, respectively.
The bold uppercase letters (\eg $\mathbf{X}$) denote matrices, while the bold lowercase letters (\eg $\mathbf{x}$) denote vectors.
We use $(j, k)$ to denote the 2D indices or locations in a matrix, and similarly, $(j)$ to denote the 1D indices of a vector. 

\subsection{Preliminary}
Non-local image modeling has demonstrated promising performance for IR tasks by exploiting NSS properties. 
Classic non-local methods, \eg the well-known non-local means (NLM) algorithm~\cite{buades2005non}, exploit the correlation among neighbouring patches with adaptive weighting. 
For each reference patch $\mathbf{X}_i \in \mathbb{R}^{\sqrt{n} \times \sqrt{n}}$ from an image $\mathbf{X} \in \mathbb{R}^{\sqrt{N} \times \sqrt{N}}$, 
NLM computes the weighting coefficients with respect to other patches as follows
\begin{align}
    \nonumber \mathbf{s}_i (j) = \Theta_i^{-1} \text{exp}( -\left \|  \mathbf{X}_j-\mathbf{X}_i \right \|_{F}^2 / b^2 )\;, \;\; \\
    \Theta_i = \sum_{j} \text{exp}(-\left \|  \mathbf{X}_j-\mathbf{X}_i \right \|_{F}^2 / b^2 ) \;\;\; \forall i\, ,
    \label{equ:nlm}
\end{align}
where $b$ is a hyper-parameter scaling the distance metric.
The approach recovers each $\mathbf{X}_i$ using a weighted average of non-local patches as $\sum_{j} \mathbf{s}_i (j) \mathbf{X}_j$.

More recent non-local works apply block matching (BM), which only selects up to $K$ ($K \ll N$) most similar patches to $\mathbf{X}_i$,
with the indices of the similar patches stored in the set $\hat{\mathcal{S}}_i$, $| \hat{\mathcal{S}}_i | \leq K$.
BM is computed by solving
\begin{align}
    \nonumber \hat{\mathcal{S}}^E_i &= \mathop{\arg\min}_{ \mathcal{S}_i:| \mathcal{S}_i | \leq K} \sum _{j \in \mathcal{S}_i} \, d_{E} (\mathbf{X}_j, \, \mathbf{X}_i)  \\
    &= \mathop{\arg\min}_{ \mathcal{S}_i:| \mathcal{S}_i | \leq K } \sum _{j \in \mathcal{S}_i} \|\mathbf{X}_j-\mathbf{X}_i\|_F^2 \;\;\;\; \forall i\;.
    \label{equ:bm}
\end{align}
Here the Euclidean distance serves as the similarity metric. 
Other functions have also been used in different algorithms, \eg transform-domain distance ~\cite{dabov2007image}, or the normalized cross correlation~\cite{wei2008fast}. The latter for the case of (non-negative) natural images leads to
\begin{align}
    \nonumber  \hat{\mathcal{S}}^{NCC}_i  &= \mathop{\arg\min}_{ \mathcal{S}_i:| \mathcal{S}_i | \leq K } \sum _{j \in \mathcal{S}_i} \, d_{NCC} (\mathbf{X}_j, \, \mathbf{X}_i)  \\
    & =  \mathop{\arg\min}_{ \mathcal{S}_i:| \mathcal{S}_i | \leq K } \, - \, \sum _{j \in \mathcal{S}_i} \frac{\text{vec}(\mathbf{X}_j)^T \text{vec}(\mathbf{X}_i)\, }{\| \mathbf{X}_j \|_F \, \| \mathbf{X}_i\|_F } \;\;\;\; \forall i\;.
    \label{equ:cor}
\end{align}
For BM-based non-local methods, a denoiser (\eg based on low-rank or sparse modeling) is applied to each patch group $\{\mathbf{X}_j\}_{j \in \hat{\mathcal{S}}_i}$ to exploit the group-wise similarity for image recovery.

\subsection{A Unified Non-Local Modeling via Self-Convolution}

Though there are various approaches to exploit image non-local self-similarity, the core of these methods is to compute the distance or similarity between each patch pair using different metrics or losses. 
We propose a novel (2D) \textbf{Self-Convolution} operator which efficiently computes the similarity between each $\mathbf{X}_i$ and the image as $\mathbf{C}_i \triangleq \mathbf{X}\, \ast \, \mathbf{X}_i \in \mathbb{R}^{\sqrt{M} \times \sqrt{M}}$, where $\sqrt{M} \triangleq (\sqrt{N} - \sqrt{n} + 1)$, which can be computed as follows~\footnote{We follow the definition of ``convolution'' in neural networks, \ie the ``correlation'' in signal processing.}
\begin{equation} \label{eq:selfConv}
\mathbf{C}_i (q, \, r) = \sum_{l = 1}^{\sqrt{n}} \, \sum_{m = 1}^{\sqrt{n}} \, \mathbf{X}_i (l, \, m) \, \mathbf{X} (q + l - 1, \, r + m - 1).
\end{equation}
Based on the Self-Convolution, we propose a general optimization problem to unify the commonly used non-local modeling methods as follows~\footnote{Self-Convolution is applied to the entire $\mathbf{X}$ in (\ref{eq:unifySelfConv}), for simplicity of analysis. Some non-local algorithms only apply BM within a neighbourhood search window $\Omega_i$ for each $\mathbf{X}_i$. The proposed Self-Convolution can be generalized to such setup using $\mathbf{X}_{\Omega_i} \, \ast \, \mathbf{X}_i$, where $\mathbf{X}_{\Omega_i}$ denotes the sub-image containing all patches within $\Omega_i$.}
\begin{equation}
    \hat{\mathbf{a}}_i \, = \,  \mathop{\argmin}_{\mathbf{a}_i} \,  \| H (\mathbf{C}_i, \mathbf{X}, \mathbf{X}_i) - \mathbf{a}_i\|_F^2 + \rho \, (\mathbf{a}_i) \;\;\; \forall i\;.
    \label{eq:unifySelfConv}
\end{equation}
Here, the penalty $\rho \, (\mathbf{a}_i)$ denotes a regularizer for $\mathbf{a}_i$, and 
$H \, (\mathbf{C}_i, \mathbf{X}, \mathbf{X}_i) \in \mathbb{R}^{M}$ 
denotes the distance/similarity measuring function based on Self-Convolution that also vectorizes the similarity map. 
Both $\rho \, (\mathbf{a}_i)$ and $H \, (\mathbf{C}_i, \mathbf{X}, \mathbf{X}_i)$ vary for different non-local modeling approaches. 

Next, we prove that the proposed Self-Convolution formulation (\ref{eq:unifySelfConv}) can be made equivalent to various commonly used non-local image operations.
To facilitate the analysis, we define a normalization vector $\mathbf{h}_{X_i} = \| \mathbf{X}_i\|_F \mathbf{h}_X$ for each $\mathbf{X}_i$, where $\mathbf{h}_X \in \mathbb{R}^{M}$ has entries $\mathbf{h}_X (j) \triangleq  \| \mathbf{X}_j\|_F$.
We let $\odot$ denote the element-wise product of two vectors or matrices.

\begin{lemma} \label{lemma:convCross}
The normalized cross-correlation in (\ref{equ:cor}) is $ d_{NCC}(\mathbf{X}_j, \, \mathbf{X}_i) = - \begin{bmatrix}\text{vec}(\mathbf{C}_i) \odot \mathbf{g}_{X_i} \end{bmatrix} \, (j)$,
where $\mathbf{g}_{X_i}(j)=1/\mathbf{h}_{X_i}(j)$ $\forall j$.
\end{lemma}
\vspace{-0.15in}

\begin{lemma} \label{lemma:convE}
The Euclidean distance in (\ref{equ:bm}) is $d_E (\mathbf{X}_j, \, \mathbf{X}_i) = \begin{bmatrix} \mathbf{h}_X \odot \mathbf{h}_X + \| \mathbf{X}_i\|_F^2 \mathbf{1} - 2\text{vec}(\mathbf{C}_i) \end{bmatrix} \, (j)$   $\forall j$.
\end{lemma}
Lemmas~\ref{lemma:convCross} and \ref{lemma:convE} show the relationship of Self-Convolution metric $H$ to various BM metrics.
We then present our main results on Self-Convolution as the following. 

\begin{figure*}[!t]
\centering 
\begin{minipage}[b]{0.9\textwidth}
\centering 
\includegraphics[width=1.0\textwidth]{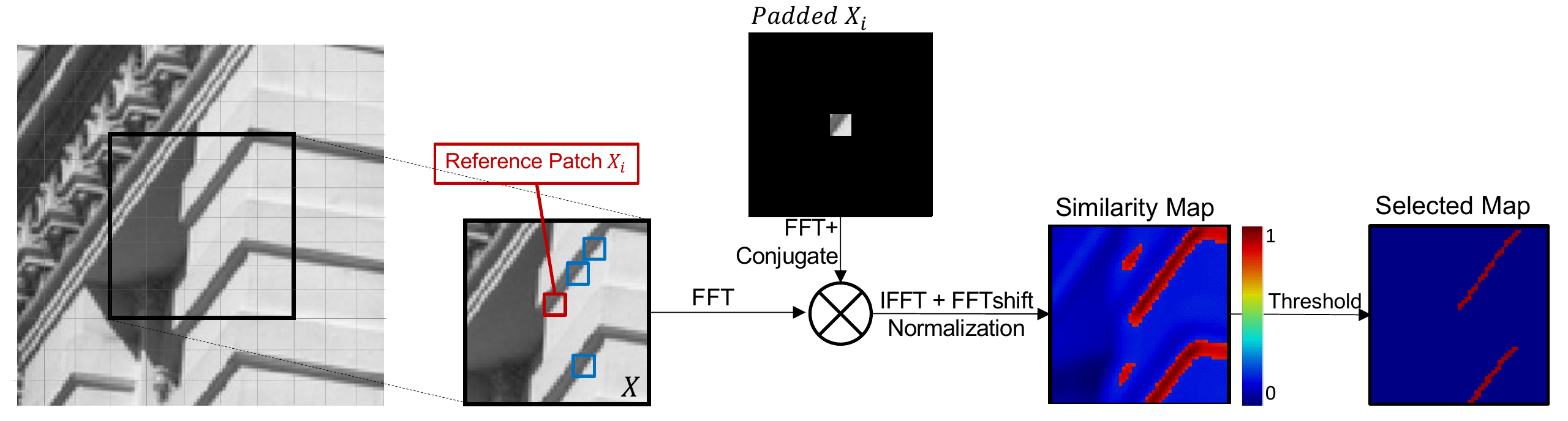}
\caption{Illustration of the proposed Self-Convolution scheme using FFTs, where IFFT denotes inverse FFT. }
\label{fig:self-conv}
\end{minipage}
\end{figure*}

\begin{proposition} \label{prop:nlm}
Self-Convolution formulation in (\ref{eq:unifySelfConv}) is equivalent to the weight calculation in NLM (\ref{equ:nlm}) if 
$\rho \, (\mathbf{a}_i) = 0$ and  $H \, (\mathbf{C}_i, \mathbf{X}, \mathbf{X}_i) = \frac{1}{\Theta_i} \text{exp}\{(2\text{vec}(\mathbf{C}_i) - \mathbf{h}_X \odot \mathbf{h}_X - \| \mathbf{X}_i\|_F^2 \mathbf{1} )/{b^2} \}$, where $\Theta_i = \sum_{j} \text{exp}(2\text{vec}(\mathbf{C}_i) - \mathbf{h}_X \odot \mathbf{h}_X - \| \mathbf{X}_i\|_F^2 \mathbf{1} )/{b^2}\} \; \forall i$.
Thus, the result of (\ref{eq:unifySelfConv}) becomes the NLM weighting, \ie $\hat{\mathbf{a}}_i = \hat{\mathbf{s}}_i$ coincides with (\ref{equ:nlm}).
\end{proposition}

\begin{theorem} \label{theorem:BM}
Self-Convolution formulation in (\ref{eq:unifySelfConv}) equivalent to block matching when 
\begin{enumerate}
    \item $\rho \, (\mathbf{a}_i) = \Phi_K(\mathbf{a}_i)$, where $\Phi_K = 0$ if $\| \mathbf{a}_i \|_0 \leq K$ and $\Phi_K = + \infty$ otherwise.
    \item $H \, (\mathbf{C}_i, \mathbf{X}, \mathbf{X}_i) = \text{vec}(\mathbf{C}_i) - \frac{1}{2} \mathbf{h}_X \odot \mathbf{h}_X$ using Euclidean distance; Or
    \item $H \, (\mathbf{C}_i, \mathbf{X}, \mathbf{X}_i) = \text{vec}(\mathbf{C}_i) \odot \mathbf{g}_{X_i}$ using normalized cross correlation.
\end{enumerate}
In particular, the support of the solution to (\ref{eq:unifySelfConv}), \ie $\text{supp} (\hat{\mathbf{a}}_i)$ is the solution $\hat{S}_i^E$ in (\ref{equ:bm}) or $\hat{S}_i^{NCC}$ in (\ref{equ:cor}).
\end{theorem}
Proposition~\ref{prop:nlm} and Theorem~\ref{theorem:BM} show the equivalence of the Self-Convolution formulation (\ref{eq:unifySelfConv}) to NLM and BM, and how (\ref{eq:unifySelfConv}) generalizes the popular non-local image modeling schemes using Self-Convolution. For the above results, we provide the full proofs in Sec.~\ref{sec:proof}.

\subsection{Efficient Self-Convolution via FFT}

Directly computing Self-Convolution in the spatial domain using (\ref{eq:selfConv}) is as expensive as the conventional BM or NLM operations. 
Instead, we propose to apply the Fast Fourier Transform (FFT) to accelerate Self-Convolution operations.
Let $\mathbf{F}\in \mathbb{C}^{N\times N}$ and $\mathbf{F}^H$ denote the 2D Fourier matrix and the inverse FFT (IFFT) matrix, respectively, and $\mathbf{x} = \text{vec}(\mathbf{X}) \in \mathbb{R}^N$.
Besides, we use $P(\cdot) :\mathbb{R}^{\sqrt{n} \times \sqrt{n}} \mapsto \mathbb{R}^{\sqrt{N} \times \sqrt{N}}$ to denote a zero-padding operator, \ie it pads zeros around a small patch to match the image size. 
We denote $\mathbf{x}_i = \text{vec}(P ( \mathbf{X}_i) ) \in \mathbb{R}^N$ as the vectorized image patch after padding, and $\bar{\mathbf{u}}$ denotes the complex conjugate of a vector $\mathbf{u}$. 

\begin{proposition} \label{prop:fft}
Self-Convolution using (\ref{eq:selfConv}) is equivalent to the following FFT-based operations (we use $\bar{\mathbf{u}}$ to denote the complex conjugate of $\mathbf{u}$):
\begin{equation} \label{eqn:fftConv}
    \mathbf{C}_i = \mathbf{F}^H(\mathbf{F}\mathbf{x} \odot  \overline{\mathbf{F}\mathbf{x}_i})\;.
\end{equation}
\end{proposition}

Comparing to (\ref{eq:selfConv}), the FFT-based Self-Convolution (\ref{eqn:fftConv}) only computes an image-level element-wise product once, which is more efficient.
Furthermore, as an additional advantage to the computation in (\ref{eq:selfConv}), we pre-computed $\mathbf{F}\mathbf{x}$ as it is fixed for calculating different $\mathbf{C}_i$ corresponding to different reference patches $\mathbf{X}_i$ within the same image.
Fig.~\ref{fig:self-conv} shows an example of the FFT-based Self-Convolution: the magnitudes of $\mathbf{C}_i$'s after normalization (Theorem~\ref{theorem:BM}) reflect the similarity between $\mathbf{X}_i$ and all non-local patches.
Based on Proposition~\ref{prop:fft}, the non-local modeling framework, which is equivalent to (\ref{eq:unifySelfConv}), using the FFT-based Self-Convolution for $\forall i$ can be formulated as
\begin{align}
\hat{\mathbf{a}}_i =  \mathop{\arg\min}_{\mathbf{a}_i} \left\| H \, ( \mathbf{F}^H(\mathbf{F}\mathbf{x} \odot  \overline{\mathbf{F}\mathbf{x}_i}), \mathbf{X}, \mathbf{X}_i ) \, - \, \mathbf{a}_i\right\|_{F}^{2} &+\rho \, (\mathbf{a}_i).
\label{equ: fft self-convn}
\end{align}

\subsection{Computational Cost Analysis} \label{sec:computation}
We assume overlapping image patches are extracted with a patch stride of $1 \times 1$. Thus, the total number of patches is approximately equal to the number of image pixels $N$. Let $n$ denote the patch size (i.e., number of pixels in a patch) and let $N_s$ denote the search window size. The computational cost of the conventional BM algorithm scales as $\mathcal{O}(nNN_s)$, while the direct convolution also requires $\mathcal{O}(nNN_s)$ operations. The proposed FFT-based Self-Convolution includes the following costs: $N$ FFTs, one for each search window, that scale as $\mathcal{O}(NN_slogN_s)$; $N$ FFTs, one for each of the zero-padded (sparse) image patches ($\mathbf{x}_i$), which require $\mathcal{O}(NnlogN_s)$ operations ~\cite{indyk2014sample}; $N$ pointwise vector products in the frequency domain requiring $\mathcal{O}(NN_s)$ multiplies; and  $N$ inverse FFTs (one for each $\mathbf{x}_i$) which scale as $\mathcal{O}(NN_slogN_s)$. 
The overall complexity is thus $\mathcal{O}\left( NN_slogN_s+(nlogN_s+N_s+N_slogN_s)N\right)$, which scales as $\mathcal{O}\left((2logN_s + 1)NN_s\right)$. Table~\ref{tab:complexity} summarizes the computational costs 
of different methods. 
As typically $2logN_s + 1 < n$, thus the FFT-based Self-Convolution is more efficient.
For example, for a single-channel image, typical $N_s$ values range from $20 \times 20$ to $30 \times 30$ and $n$ typically ranges from $6\times6$ to $10\times10$. Thus, $2logN_s+1$ is around $20$, which is much smaller than the patch sizes $36$ or $100$.
The gap is larger when the algorithm works with 3D or higher dimensional patches. For example, in video or multi-modality data processing problems, $n$ can be $6 \times 6 \times 4$ or larger, and the FFT-based Self-Convolution has much lower computational cost than other schemes.

\begin{table}[!t]
\centering
\footnotesize
\setlength{\tabcolsep}{0.2em}
\adjustbox{width=1.\linewidth}{
    \begin{tabular}{l|c}
        \toprule
        Operation Type &  Computational Cost \\ \midrule
        Block Matching & $\mathcal{O}(nNN_s)$\\
        Self-Conv (\ref{eq:selfConv}) & $\mathcal{O}(nNN_s)$ \\
       \rowcolor{Gray} Self-Conv (FFT) & $\mathcal{O}\left(N N_s log N_s  + (nlogN_s+N_s+N_slogN_s)N\right)$ \\ 
        \bottomrule
    \end{tabular}
}
\vspace{0.1in}
   \caption{Computational costs for various methods (over all reference patches). Here, $n$ denotes the number of pixels in a patch, $N_s$ denotes the number of patches within the search window, and $N$ denotes the number of patches within the whole image.}\label{tab:complexity}
\end{table}

\subsection{Proofs for Unified Non-Local Modeling via Self-Convolution} \label{sec:proof}
\setcounter{lemma}{0}
\setcounter{proposition}{0}
\setcounter{theorem}{0}
We present the detailed proofs regarding the proposed Self-Convolution formulation (\ref{eq:unifySelfConv}), where the Self-Convolution $\mathbf{C}_i$ is computed as in (\ref{eq:selfConv}).

First of all, we prove Lemma~\ref{lemma:convCross_app} which shows the connection between (\ref{eq:selfConv}) and the normalized cross correlation~\cite{wei2008fast}. 

\begin{lemma} \label{lemma:convCross_app}
\vspace{0.1in}
The normalized cross-correlation in (\ref{equ:cor}) is $ d_{NCC}(\mathbf{X}_j, \, \mathbf{X}_i) = - \begin{bmatrix}\text{vec}(\mathbf{C}_i) \odot \mathbf{g}_{X_i} \end{bmatrix} \, (j)$,
where $\mathbf{g}_{X_i}(j)=1/\mathbf{h}_{X_i}(j)$ $\forall j$.
\end{lemma}
\vspace{-0.15in}
\begin{proof}[Proof of Lemma~\ref{lemma:convCross_app}]
By converting the 2D coordinates $(q, r)$ used in (\ref{eq:selfConv}) into 1D as $(j)$, and vectorizing the reference patch $\mathbf{X}_i$ and the Self-Convolution matrix $\mathbf{C}_i$, (\ref{eq:selfConv}) is equivalent to
\begin{align}\label{eq:vecCor_app}
    \nonumber \text{vec}(\mathbf{C}_i) (j) &= \sum_{l = 1}^{n} \text{vec}(\mathbf{X}_i) (l) \, \times \,  \text{vec}(\mathbf{X}_j) (l) \\
    & =  \text{vec}(\mathbf{X}_i)^T \text{vec}(\mathbf{X}_{j})\;.
\end{align}
where $\mathbf{X}_{j}$ denotes the $j$-th patch (with its top left corner at pixel $(q,r)$ as in (\ref{eq:selfConv})) that is extracted from $\mathbf{X}$.
Therefore, the normalized cross-correlation defined by (\ref{equ:cor}) is
\begin{align}
    \nonumber d_{NCC} (\mathbf{X}_j, \, \mathbf{X}_i)  & =      -\frac{\text{vec}(\mathbf{X}_j)^T \text{vec}(\mathbf{X}_i)\, }{\| \mathbf{X}_j \|_F \, \| \mathbf{X}_i\|_F } \; 
    = - \frac{\text{vec}(\mathbf{C}_i)(j)}{\mathbf{h}_{X_i}(j)}\\
    & = - \begin{bmatrix}\text{vec}(\mathbf{C}_i) \odot \mathbf{g}_{X_i} \end{bmatrix} \, (j) \;,
\end{align}
where $\mathbf{g}_{X_i}(j)=1/\mathbf{h}_{X_i}(j)$ $\forall j$.
\vspace{0.05in}
\end{proof}

Next, we prove the Lemma~\ref{lemma:convE_app} which shows the connection between (\ref{eq:selfConv}) and the Euclidean distance defined as
\begin{equation}
   d_{E} (\mathbf{X}_j, \, \mathbf{X}_i)  =  \|\mathbf{X}_j-\mathbf{X}_i\|_F^2\;.
    \label{equ:bmMetric_app}
\end{equation}

\begin{lemma} \label{lemma:convE_app}
\vspace{0.1in}
The Euclidean distance in (\ref{equ:bmMetric_app}) is $d_E (\mathbf{X}_j, \, \mathbf{X}_i) = \begin{bmatrix} \mathbf{h}_X \odot \mathbf{h}_X + \| \mathbf{X}_i\|_F^2 \mathbf{1} - 2\text{vec}(\mathbf{C}_i) \end{bmatrix} \, (j)$   $\forall j$\;.
\end{lemma}
\begin{proof}[Proof of Lemma~\ref{lemma:convE_app}]
The Euclidean distance defined by (\ref{equ:bmMetric_app}) is
\begin{align} \label{eq:EucConv_app}
    \nonumber & d_{E} (\mathbf{X}_j, \, \mathbf{X}_i) =  \|\mathbf{X}_j-\mathbf{X}_i\|_F^2 \\
    & = \|\mathbf{X}_j\|_F^2 + \|\mathbf{X}_i\|_F^2 - 2\text{vec}(\mathbf{X}_i)^T \text{vec}(\mathbf{X}_j) \;.
\end{align}
Here, $\|\mathbf{X}_j\|_F^2 = \begin{bmatrix} \mathbf{h}_X \odot \mathbf{h}_X \end{bmatrix}(j)$.
Besides, according to (\ref{eq:vecCor_app}) from the proof for Lemma~\ref{lemma:convCross_app}, we have $2\text{vec}(\mathbf{X}_i)^T \text{vec}(\mathbf{X}_j) = 2\text{vec}(\mathbf{C}_i)(j)$.
Therefore, the Euclidean distance is
\begin{align}
    \nonumber d_{E} (\mathbf{X}_j, \, \mathbf{X}_i)  = & \begin{bmatrix} \mathbf{h}_X \odot \mathbf{h}_X \end{bmatrix}(j) + \|\mathbf{X}_i\|_F^2 - 2\text{vec}(\mathbf{C}_i)(j) \\
    = & \begin{bmatrix} \mathbf{h}_X \odot \mathbf{h}_X + \| \mathbf{X}_i\|_F^2 \mathbf{1} - 2\text{vec}(\mathbf{C}_i) \end{bmatrix} \, (j)\; .
\end{align}
\end{proof}

We now prove Proposition~\ref{prop:nlm_app} which shows the equivalence between the Self-Convolution Formulation (\ref{eq:unifySelfConv}) and the non-local means (NLM) algorithm, which computes the weighting coefficients with respect to other patches as follows:
\begin{align}
    \nonumber \mathbf{s}_i (j) = \Theta_i^{-1} \text{exp}( -\left \|  \mathbf{X}_j-\mathbf{X}_i \right \|_{F}^2 / b^2 )\;, \;\; \\
    \Theta_i = \sum_{j} \text{exp}(-\left \|  \mathbf{X}_j-\mathbf{X}_i \right \|_{F}^2 / b^2 ) \;\;\; \forall i\, ,
    \label{equ:nlm_app}
\end{align}
where $b$ is a hyper-parameter scaling the distance metric.

\begin{proposition} \label{prop:nlm_app}
Self-Convolution formulation in (\ref{eq:unifySelfConv}) is equivalent to the weight calculation in NLM (\ref{equ:nlm_app}) if 
$\rho \, (\mathbf{a}_i) = 0$ and  $H \, (\mathbf{C}_i, \mathbf{X}, \mathbf{X}_i) = \frac{1}{\Theta_i} \text{exp}\{(2\text{vec}(\mathbf{C}_i) - \mathbf{h}_X \odot \mathbf{h}_X - \| \mathbf{X}_i\|_F^2 \mathbf{1} )/{b^2} \}$\;.
Thus, the result of (\ref{eq:unifySelfConv}) becomes the NLM weighting, \ie $\hat{\mathbf{a}}_i = \hat{\mathbf{s}}_i$ coincides with (\ref{equ:nlm_app}).
\end{proposition}
\begin{proof}[Proof of Proposition~\ref{prop:nlm_app}]
The Euclidean distance in the exponent in (\ref{equ:nlm_app}) can be rewritten using Lemma~\ref{lemma:convE_app}. Then, with $\rho \, (\mathbf{a}_i) = 0$ and  $H \, (\mathbf{C}_i, \mathbf{X}, \mathbf{X}_i) = \frac{1}{\Theta_i} \text{exp}\{(2\text{vec}(\mathbf{C}_i) - \mathbf{h}_X \odot \mathbf{h}_X - \| \mathbf{X}_i\|_F^2 \mathbf{1} )/{b^2} \}$ ($\text{exp}\{\cdot\}$ denotes element-wise exponential function), the Self-Convolution formulation (\ref{eq:unifySelfConv}) reduces to
$ \hat{\mathbf{a}}_i \, =
     \argmin_{\mathbf{a}_i} \,  \| \frac{1}{\Theta_i} \text{exp}\{(2\text{vec}(\mathbf{C}_i) - \mathbf{h}_X \odot \mathbf{h}_X - \| \mathbf{X}_i\|_F^2 \mathbf{1} )/{b^2} \} - \mathbf{a}_i\|_F^2 \; \forall i\;,
$
which has an exact solution as  $\hat{\mathbf{a}}_i = \frac{1}{\Theta_i} \text{exp}\{(2\text{vec}(\mathbf{C}_i) - \mathbf{h}_X \odot \mathbf{h}_X - \| \mathbf{X}_i\|_F^2 \mathbf{1} )/{b^2} \}$\;.
Based on Lemma~\ref{lemma:convE_app},  $\hat{\mathbf{a}}_i(j) = \frac{1}{\Theta_i} \text{exp}( -\left \|  \mathbf{X}_j-\mathbf{X}_i \right \|_{F}^2 / b^2 )$, which completes the proof.
\end{proof}

Finally, we prove Theorem~\ref{theorem:BM_app} that shows the equivalence between the Self-Convolution Formulation (\ref{eq:unifySelfConv}) and the block matching (BM) algorithm.
The BM algorithm using Euclidean distance metric is defined in (\ref{equ:bm}). 
Besides, the BM algorithm using normalized cross correlation as the metric is defined as in (\ref{equ:cor}).

\begin{theorem} \label{theorem:BM_app}
Self-Convolution formulation is equivalent to block matching when 
\begin{enumerate}
    \item $\rho \, (\mathbf{a}_i) = \Phi_K(\mathbf{a}_i)$, where $\Phi_K(\mathbf{a}_i) = 0$ if $\| \mathbf{a}_i \|_0 \leq K$ and $\Phi_K(\mathbf{a}_i) = + \infty$ otherwise.
    \item $H \, (\mathbf{C}_i, \mathbf{X}, \mathbf{X}_i) = \text{vec}(\mathbf{C}_i) - \frac{1}{2} \mathbf{h}_X \odot \mathbf{h}_X$ using Euclidean distance; Or
    \item $H \, (\mathbf{C}_i, \mathbf{X}, \mathbf{X}_i) = \text{vec}(\mathbf{C}_i) \odot \mathbf{g}_{X_i}$ using normalized cross correlation.
\end{enumerate}
Thus, the support of the solution to (\ref{eq:unifySelfConv}), \ie $\text{supp} (\hat{\mathbf{a}}_i)$ is the solution $\hat{S}_i^E$ in (\ref{equ:bm}) or $\hat{S}_i^{NCC}$ in (\ref{equ:cor}).
\end{theorem}
\begin{proof}[Proof of Theorem~\ref{theorem:BM_app}]
First, when $\rho \, (\mathbf{a}_i) = \Phi_K(\mathbf{a}_i)$ and $H \, (\mathbf{C}_i, \mathbf{X}, \mathbf{X}_i) = \text{vec}(\mathbf{C}_i) - \frac{1}{2} \mathbf{h}_X \odot \mathbf{h}_X$, the Self-Convolution formulation is reduced to
\begin{equation}
    \hat{\mathbf{a}}_i \, = \,  \mathop{\argmin}_{\mathbf{a}_i} \,  \| \text{vec}(\mathbf{C}_i) - \frac{1}{2} \mathbf{h}_X \odot \mathbf{h}_X - \mathbf{a}_i\|_F^2 + \Phi_K(\mathbf{a}_i) \;\;\; \forall i\;.
\end{equation}
As the penalty term $\Phi_K(\mathbf{a}_i)$ is a barrier function, it is equivalent to imposing an $\ell_0$ sparsity constraint. Therefore, the equivalent constrained optimization problem is
\begin{align} \label{eq:BMsparsity_app}
    \nonumber \hat{\mathbf{a}}_i \, = & \,  \mathop{\argmin}_{\mathbf{a}_i} \,  \| \text{vec}(\mathbf{C}_i) - \frac{1}{2} \mathbf{h}_X \odot \mathbf{h}_X - \mathbf{a}_i\|_F^2\\
    & \;\;\;\;\;\;\;\;\;\st\;\| \mathbf{a}_i \|_0 \leq K \;\;\forall i\;.
\end{align}
Denote $\mathbf{b}_i \triangleq \text{vec}(\mathbf{C}_i) - \frac{1}{2} \mathbf{h}_X \odot \mathbf{h}_X$. The exact solution to (\ref{eq:BMsparsity_app}) is obtained by projecting $\mathbf{b}_i$ onto the $\ell_0$ ball. Thus the optimal $\hat{\mathbf{a}}_i$ can be computed as
\begin{equation} \label{sol:balproj_app}
 \hat{\mathbf{a}}_i (j) = \left \{ \begin{matrix}
 \mathbf{b}_i(j) & , \;\; j \in \Omega_K^i \\
 0  & ,\;\; j \notin \Omega_K^i
\end{matrix}\right .
\end{equation}
Here the set $\Omega_K^i = \text{supp}(\hat{\mathbf{a}}_i)$ indexes the top-K elements of largest magnitude in $\mathbf{b}_i$, which are identical to the solution to (\ref{eq:suppOmega_app}) below.
\begin{equation}\label{eq:suppOmega_app}
    \Omega_K^i = \mathop{\argmax}_{ \mathcal{S}_i:| \mathcal{S}_i | \leq K} \sum _{j \in \mathcal{S}_i} \mathbf{b}_i(j)
\end{equation}
Furthermore, based on Lemma~\ref{lemma:convE_app}, we have
\begin{align} \label{eq:bEuclid_app}
    \nonumber \mathbf{b}_i(j) = & \text{vec}(\mathbf{C}_i) - \frac{1}{2}\mathbf{h}_X \odot \mathbf{h}_X \\
    = & - d_{E} (\mathbf{X}_j, \, \mathbf{X}_i) + \frac{1}{2}\| \mathbf{X}_i\|_F^2 \mathbf{1}\;.
\end{align}
Substituting (\ref{eq:bEuclid_app}) to (\ref{eq:suppOmega_app}), we prove the equivalence using \textbf{Euclidean distance} as
\begin{align}
    \nonumber  \text{supp}(\hat{\mathbf{a}}_i) & = \mathop{\argmax}_{ \mathcal{S}_i:| \mathcal{S}_i | \leq K} \sum _{j \in \mathcal{S}_i} \left \{ - d_{E} (\mathbf{X}_j, \, \mathbf{X}_i) + \frac{1}{2}\| \mathbf{X}_i\|_F^2 \mathbf{1} \right \} \\
    & = \mathop{\argmin}_{ \mathcal{S}_i:| \mathcal{S}_i | \leq K} \sum _{j \in \mathcal{S}_i} d_{E} (\mathbf{X}_j, \, \mathbf{X}_i)\;.
\end{align}

Next, when $\rho \, (\mathbf{a}_i) = \Phi_K(\mathbf{a}_i)$ and $H \, (\mathbf{C}_i, \mathbf{X}, \mathbf{X}_i) = \text{vec}(\mathbf{C}_i) \odot \mathbf{g}_{X_i}$, the Self-Convolution formulation (\ref{eq:unifySelfConv}) reduces to the following form
\begin{align} \label{eq:BMcorr_app}
    \nonumber \hat{\mathbf{a}}_i \, = & \,  \mathop{\arg\min}_{\mathbf{a}_i} \,  \| \text{vec}(\mathbf{C}_i) \odot \mathbf{g}_{X_i} - \mathbf{a}_i\|_F^2 \\
    & \;\;\;\;\;\;\;\;\;\st\;\| \mathbf{a}_i \|_0 \leq K \;\;\forall i\;.
\end{align}
Similarly as in (\ref{sol:balproj_app}), there is an exact solution to (\ref{eq:BMcorr_app}) by setting $\mathbf{b}_i \triangleq \text{vec}(\mathbf{C}_i) \odot \mathbf{g}_{X_i}$ in (\ref{sol:balproj_app}).
Based on Lemma~\ref{lemma:convCross_app} and (\ref{eq:suppOmega_app}), we prove the equivalence using \textbf{normalized cross correlation} as
\begin{align}
\nonumber \text{supp}(\hat{\mathbf{a}}_i) =& \mathop{\argmax}_{ \mathcal{S}_i:| \mathcal{S}_i | \leq K} \sum _{j \in \mathcal{S}_i} \begin{bmatrix}\text{vec}(\mathbf{C}_i) \odot \mathbf{g}_{X_i} \end{bmatrix} \, (j) \\
    =& \mathop{\argmin}_{ \mathcal{S}_i:| \mathcal{S}_i | \leq K} \sum _{j \in \mathcal{S}_i} - \begin{bmatrix}\text{vec}(\mathbf{C}_i) \odot \mathbf{g}_{X_i} \end{bmatrix} \, (j)\;.
\end{align}

\end{proof}

\section{Self-Convolution for Multi-Modality Image Denoising}

\begin{figure*}[!t]
\centering 
\begin{minipage}[b]{1.0\textwidth}
\centering 
\includegraphics[width=1.\textwidth]{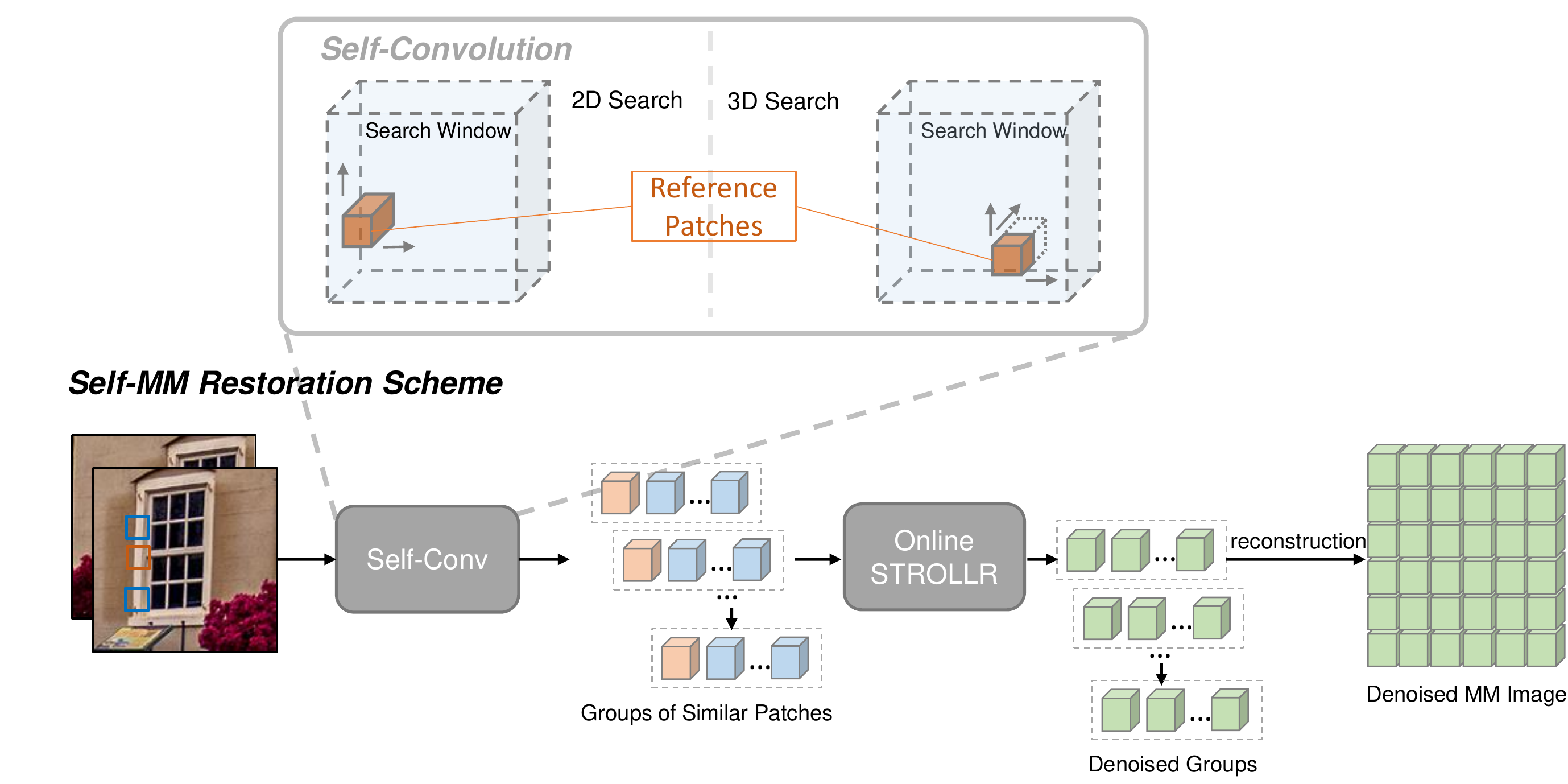}
\caption{Illustration of the online Self-MM scheme for multi-modality (MM) image restoration, which can be divided into three steps: (1) self-convolution for grouping similar patches, including 2D search and 3D search, (2) online STROLLR denoising, (3) image reconstruction.}
\label{fig:self-MM}
\end{minipage}
\end{figure*}

We propose a multi-modality image denoising scheme, called online Self-MM, which has an efficient algorithm using FFT-based Self-Convolution.
Fig.~\ref{fig:self-MM} shows the algorithm pipeline, which contains three steps, namely (1) Self-Convolution, (2) Online Denoising, and (3) Image Reconstruction. 

\textbf{Self-Convolution}: A multi-modality image can be represented as a tensor $\mathcal{X} \in \mathbb{R}^{\sqrt{N} \times \sqrt{N} \times p}$, and the proposed algorithm aims to restore its noisy version $\mathcal{Y}$.
We apply \textbf{2D} Self-Convolution to $\mathcal{Y}$, \ie 2D convolution along each spatial channel followed by summing the result over channels, working with the overlapping 3D patches $\mathcal{Y}_i \in \mathbb{R}^{\sqrt{n} \times \sqrt{n} \times p}$ which are extracted from $\mathcal{Y}$. 
The noisy tensor and padded 3D patch are vectorized as $\mathbf{y}=\text{vecT}(\mathcal{Y})$ and $\mathbf{y}_i = \text{vecT}(P(\mathcal{Y}_i)) \in \mathbb{R}^{Np}$, and the \textbf{2D} Self-Convolution is obtained by solving the following problem: 
\begin{align}
\label{equ:self-convMM}
\nonumber \hat{\mathbf{a}}_i = & \mathop{\argmin}_{\mathbf{a}_i} \left\|  S(\mathbf{F}^H(\mathbf{F}\mathbf{y} \odot  \overline{\mathbf{F}\mathbf{y}_i})) - 0.5 \, \mathbf{h}_\mathcal{Y} \odot \mathbf{h}_\mathcal{Y}  \, - \, \mathbf{a}_i\right\|_{F}^{2}  \\
 & \;\;\;\;\;\;\;\;\;\;\;\;\;\;\;\;\;\;\;\;\;\;\; \st \; \left \| \mathbf{a}_i \right \|_0 \leq K \;\;\; \forall i\;,
\end{align}
where $\mathbf{F}$ denotes 2D (spatial) channel-wise FFT for a 3D input and $S(\cdot)$ performs summation along the channel dimension. 
The solution to (\ref{equ:self-convMM}) is obtained by retaining only the K-largest magnitude elements in $S(\mathbf{F}^H(\mathbf{F}\mathbf{y} \odot  \mathbf{F}\mathbf{y}_i)) - 0.5 \, \mathbf{h}_\mathcal{Y} \odot \mathbf{h}_\mathcal{Y}$, and setting the others to zero~\cite{octobos}. Alternatively, one can also apply \textbf{3D} Self-Convolution to $\mathcal{Y}$, by working with patches $\mathcal{Y}_i \in \mathbb{R}^{\sqrt{n} \times \sqrt{n} \times d}$ where $d < p$. In this case, Self-Convolution allows to search across channels by applying a 3D FFT (and omitting the $S(\cdot)$ function) in (\ref{equ:self-convMM}). We construct the groups of similar patches $\mathbf{Y}_i$ using $\hat{\mathcal{S}}_i = \text{supp} (\hat{\mathbf{a}}_i)$~\footnote{The first element is reference patch since it has a distance of zero from itself.} as
\begin{equation}
    \mathbf{Y}_i = \begin{bmatrix}  \mathbf{y}_{ \hat{\mathcal{S}}_i (1) }, \mathbf{y}_{\hat{\mathcal{S}}_i (2)}, ... , \mathbf{y}_{\hat{\mathcal{S}}_i (K)} \end{bmatrix} \;\;\; \forall i\;.
\end{equation}

\begin{algorithm}[!t]
\caption{The online Self-MM Multi-modality Image recovery algorithm framework.}
\begin{algorithmic}
\REQUIRE The observed noisy image $\mathcal{Y}$. \\
\textbf{Self-Convolution:}
From $\mathcal{Y}$, obtain the similar patches sequence $\{\mathbf{Y}_t\}_{t=1}^N$.
\\
\textbf{Initialize:} $\hat{\mathbf{W}}_{0} = \mathbf{W}_{0}$ (\eg 4D DCT), $\mathbf{V}_0=0$, and $\mathbf{Z}_t=\mathbf{Y}_t \; \forall t$. \\
\textbf{For $\;t = 1, 2,\ldots, N$ Repeat}
\begin{enumerate}
\item \textbf{Low-rank Approximation:} 
\begin{enumerate}
\item Full SVD: $ \mathbf{\Lambda}_t \operatorname{diag}(\omega_t) \mathbf{\Delta}_{t}^{T} \leftarrow \text{SVD}(\mathbf{Z}_{t})$.
\item Update $\hat{\mathbf{D}}_{t} = \mathbf{\Lambda}_{t} \operatorname{diag}\{H_\theta (\omega_{t})\} \mathbf{\Delta}_{t}^{T}$.
\end{enumerate}
\item \textbf{Sparse Coding:} 
$\hat{\alpha}_t = H_\beta(\hat{\mathbf{W}}_{t-1} \mathbf{z}_t)$.
\item \textbf{Online Transform Update:} 
\begin{enumerate}
\item $\mathbf{V}_t=(1-{t}^{-1})\mathbf{V}_{t-1}+t^{-1}  \mathbf{z}_t \hat{\alpha}_t^T$.
\item Full SVD:  $\mathbf{\Phi}_{t} \mathbf{\Sigma}_{t} \mathbf{\Psi}_{t}^{T} \leftarrow \text{SVD}(\mathbf{V}_t)$.
\item Update $\hat{\mathbf{W}}_t= \mathbf{\Psi}_{t}\mathbf{\Phi}_{t}^{T}$
\end{enumerate}
\item \textbf{3D Denoised Patch Reconstruction:} 
\begin{enumerate}
\item Sparse coding: $\hat{\alpha}_i = H_\beta(\hat{\mathbf{W}}_{i} \mathbf{z}_i)$.
\item Reconstruct $\hat{\mathbf{Z}}_t$ by least square solution as $\hat{\mathbf{Z}}_t = \begin{bmatrix} \mathbf{Y}_t + \gamma_l \hat{\mathbf{D}}_t + \gamma_s \text{vec}^{-1} (\hat{\mathbf{W}}_t^T \hat{\alpha}_t) \end{bmatrix}  / (1 + \gamma_l + \gamma_s)$.
\item Tensorize $\mathcal{\hat{Z}}_t = \text{vecT}^{-1}(\mathbf{\hat{Z}}_t)$.
\end{enumerate}
\item \textbf{Aggregation:} Aggregate patches $\{\hat{\mathcal{Z}_t}\}$ at corresponding locations to $\mathcal{\hat{Y}}_t$.
\end{enumerate}
\textbf{End} \\
\textbf{Output:} The reconstructed (denoised) image $\mathcal{\hat{Y}}_N$.  \\
\end{algorithmic}
\end{algorithm}

\textbf{Online Denoising}: 
Recent works~\cite{wen2017strollr2d,wen2020image} proposed an effective non-local image model, called STROLLR, using \textit{transform learning}.
However, the batch STROLLR algorithm involves inefficient BM, which limits the algorithm's scalability to handle high-dimensional or multi-modality images.
For large-scale images, we need to process a large number of $\mathbf{Y}_i$'s. Thus, we propose to extend the batch STROLLR to online learning, and apply online STROLLR to process $\mathbf{Y}_\tau$ sequentially for $\tau = 1, 2, ... t, ..., N$. When denoising $\mathbf{Y}_t$, the online algorithm solves the following problem:
\begin{align} \label{equ:onlineStr}
\nonumber \{ \hat{\mathbf{Z}}_t, \hat{\mathbf{W}}_t, \hat{\alpha}_t, \hat{\mathbf{D}}_t \} & =  \mathop{\argmin}_{ \{\mathbf{Z}_t, \mathbf{W}_t, \alpha _t, \mathbf{D}_t\}}  \left \| \mathbf{Y}_t - \mathbf{Z}_t  \right \|_F^2 \\
\nonumber & \;\;\; +  \gamma_l \left \{ \left \| \mathbf{Z}_t - \mathbf{D}_t \right \|_F^2 + \theta ^2 \operatorname{rank}(\mathbf{D}_t) \right \} \\
\nonumber +& \gamma_s \sum_{\tau=1}^{t} \left \{ \left \| \mathbf{W}_t \mathbf{z}_\tau - \alpha_\tau \right \|^2_2 + \beta^2 \left \| \alpha_\tau \right \|_0  \right \} \\
& \;\;\;\quad \st \;\;\mathbf{W}_t^T \mathbf{W}_t = \mathbf{I}\; .
\end{align}
Here $\mathbf{Z}_t$ is the denoised estimate of $\mathbf{Y}_t$ with $\mathbf{z}_{\tau} = \text{vec}(\mathbf{Z}_{\tau})$, and $\gamma_l$, $\gamma_s$, $\theta$, and $\beta$ are non-negative hyper-parameters. 
In (\ref{equ:onlineStr}), $\mathbf{D}_t$ and $\alpha_t$ denote the low-rank approximation and sparse codes for $\mathbf{Z}_t$, respectively. 
The transform $\mathbf{W}_t \in \mathbb{R}^{nK \times nK}$ is learned online by adapting to all $\{ \mathbf{z}_\tau \}_{\tau = 1}^t$. 
We propose to solve (\ref{equ:onlineStr}) using an efficient block coordinate descent algorithm~\cite{wen2020image} with initialization for variables based on most recent estimates, which is similar to the batch STROLLR algorithm~\cite{wen2020image}.
First, fixing $\mathbf{Z}_t=\mathbf{Y}_t$, we solve for $\mathbf{D}_t$ in (\ref{equ:onlineStr}) by finding the best low-rank approximation to $\mathbf{Z}_t$ as follows:
\begin{equation}
    \hat{\mathbf{D}}_t=\underset{\mathbf{D}_{t}}{\argmin}\left\|\mathbf{Z}_{t}-\mathbf{D}_t\right\|_{F}^{2}+\theta^{2} \operatorname{rank}\left(\mathbf{D}_t\right) \;.
    \label{equ:low-rank approximation}
\end{equation}
Let $\mathbf{\mathbf{Z}}_t = \mathbf{\Lambda}_t \operatorname{diag}(\omega_t) \mathbf{\Delta}_{t}^{T}$ be a  full singular value decomposition (SVD), where the diagonal vector $\omega_t$ contains the singular values. Then the exact solution to (\ref{equ:low-rank approximation}) is $\hat{\mathbf{D}}_t = \mathbf{\Lambda}_{t} \operatorname{diag}\{H_\theta (\omega_{t})\} \mathbf{\Delta}_{t}^{T}$,
where $H_\theta(\cdot)$ denotes the hard thresholding operator at threshold $\theta$.
Next, with $\mathbf{W}_t$ fixed to $\hat{\mathbf{W}}_{t-1}$ and fixed $\mathbf{z}_t$, we solve the following sparse coding problem for $\hat{\alpha}_t$:
\begin{equation}
\hat{\alpha}_t = \underset{\alpha_{t}}{\operatorname{argmin}}\left\|\mathbf{W}_t \mathbf{z}_t - \alpha_t \right\|_{2}^{2} + \beta^2 \left\|\alpha_t\right\|_{0}\;,
\end{equation}
which is known as the transform model sparse coding problem, and the exact solution $\hat{\alpha}_t = H_\beta(\mathbf{W}_t \mathbf{z}_t)$. Next, the transform matrix is updated in (\ref{equ:onlineStr}) by solving the following problem:
\begin{align} \label{eq:Wupdate}
     \hat{\mathbf{W}}_t = \mathop{\argmin}_{\mathbf{W}_t} \sum_{\tau=1}^t \left\| \mathbf{W}_t \mathbf{z}_\tau - \hat{\alpha}_\tau \right\|_2^2 
    \quad \st\; \mathbf{W}_t^T\mathbf{W}_t=\mathbf{I} \;.
\end{align}
Denoting the full SVD of $\mathbf{V}_t \triangleq \sum_{\tau =1}^t \mathbf{z}_\tau \hat{\alpha}_\tau^{T}$ as $\mathbf{\Phi}_{t} \mathbf{\Sigma}_{t} \mathbf{\Psi}_{t}^{T}$, the exact solution to (\ref{eq:Wupdate}) is obtained as $\hat{\mathbf{W}}_t= \mathbf{\Psi}_{t}\mathbf{\Phi}_{t}^{T}$~\cite{ravishankar2013closed}. 
For the transform update step (\ref{eq:Wupdate}), the SVD of the matrix   $\mathbf{V}_t=(1-{t}^{-1})\mathbf{V}_{t-1}+t^{-1}  \mathbf{z}_t \hat{\alpha}_t^T$ can be efficiently updated (without storing all the past estimated variables) via a fast rank-1 update to the SVD of $\mathbf{V}_{t-1}$~\cite{wen2017joint}.
Finally, with fixed $\hat{\mathbf{D}}_t$, $\hat{\alpha}_t $, and $\hat{\mathbf{W}}_t$, the denoised patch group $\mathbf{Z}_t$ is updated by solving
\begin{align} \label{eq:Zupdate}
   \nonumber \hat{\mathbf{Z}}_t = \mathop{\argmin}_{\mathbf{Z}_t} \left \| \mathbf{Y}_t - \mathbf{Z}_t  \right \|_F^2 +  \gamma_l \left \| \mathbf{Z}_t - \hat{\mathbf{D}}_t \right \|_F^2 \\
   + \gamma_s  \left \| \hat{\mathbf{W}}_t \mathbf{z}_t - \hat{\alpha}_t \right \|^2_2 \;\;,
\end{align}
which has a least-square solution $\hat{\mathbf{Z}}_t = \begin{bmatrix} \mathbf{Y}_t + \gamma_l \hat{\mathbf{D}}_t + \gamma_s \text{vec}^{-1} (\hat{\mathbf{W}}_t^T \hat{\alpha}_t) \end{bmatrix}  / (1 + \gamma_l + \gamma_s)$.

The proposed online algorithm denoises each $\mathbf{Y}_\tau$ sequentially, and updates the sparsifying transform every time a new $\mathbf{Y}_\tau$ arrives.
Alternatively, similar to \cite{wen2017joint}, the strictly online denoising can be modified into a mini-batch algorithm based on (\ref{equ:onlineStr}) to process a set of $\mathbf{Y}_\tau$'s each time, and only update the transform once per mini-batch.
Compared to the strictly online algorithm, the mini-batch learning scheme requires higher memory usage, but enjoys a lower computational cost.

\textbf{Image Reconstruction}: 
The online denoising algorithm sequentially outputs each patch group $\hat{\mathbf{Z}}_i = \begin{bmatrix}  \hat{\mathbf{Z}}_{i}(\cdot,1), \hat{\mathbf{Z}}_{i}(\cdot,2), ... , \hat{\mathbf{Z}}_i(\cdot,K) \end{bmatrix}$, which is the denoised estimate of the input $\hat{\mathbf{Y}}_i = \begin{bmatrix}  \mathbf{y}_{ \hat{\mathcal{S}}_i (1) }, \mathbf{y}_{\hat{\mathcal{S}}_i (2)}, ... , \mathbf{y}_{\hat{\mathcal{S}}_i (K)} \end{bmatrix}$.
Correspondingly, each $\hat{\mathbf{Z}}_{i}(\cdot,j)$ is the denoised estimate of $\mathbf{y}_{\hat{\mathcal{S}}_i (j)}$, and these denoised patches are aggregated back at their respective locations in the high-dimensional image.

For the online Self-MM scheme, we first tensorize each $\hat{\mathbf{Z}}_{i}(\cdot,j)$ as $\hat{\mathcal{Z}}_{i}(\cdot,j) = \text{vecT}^{-1} (\hat{\mathbf{Z}}_{i}(\cdot,j)) \in \mathbb{R}^{\sqrt{n} \times \sqrt{n} \times p}$, and then deposit each denoised $\hat{\mathcal{Z}}_{i}(\cdot,j)$ back at the corresponding location of the image (corresponding to $\hat{S}_i(j)$).
As we are working with overlapping patches, the pixel values of the aggregated image are normalized by the number of aggregated patches containing each pixel, which then generates the final denoised MM image.

\section{Experiments}
In this section, we first demonstrate the speed-up of Self-Convolution by incorporating it in several non-local image restoration algorithms, including gray-scale image denoising, inpainting, and multi-channel denoising algorithms.
Then we evaluate the promise of the proposed online Self-MM based multi-modality image restoration framework on RGB-NIR dataset.

\subsection{Speed-Up of Non-Local Algorithms by Self-Convolution}
We first demonstrate how the proposed Self-Convolution can significantly speed up non-local IR algorithms in general.
We select popular single-channel IR algorithms, including WNNM \cite{gu2014weighted}, STROLLR \cite{wen2017strollr2d}, GHP~\cite{zuo2013texture}, NCSR~\cite{dong2012nonlocally}, SAIST~\cite{dong2012nonlocal}, and PGPD~\cite{xu2015patch} for denoising; and TSLRA~\cite{guo2017patch} and GSR~\cite{zhang2014group} for inpainting, which all involve 2D BM.
Furthermore, we select multi-channel algorithms, including MCWNNM~\cite{xu2017multi} for multi-channel image denoising, SALT~\cite{wen2017joint} for video denoising, as well as Self-MM~\cite{guo2020self} for multi-modality image denoising, which all involve high-dimensional BM.

\begin{table*}[htbp]
  \centering
  \footnotesize
  \setlength{\tabcolsep}{0.7em}
    \adjustbox{width=0.8\linewidth}{
        \begin{tabular}{cl|ccc|ccc}
        \toprule
        \multirow{2}{*}{Task} & \multirow{2}{*}{Method} & \multicolumn{3}{c|}{Vanilla Block Matching} & \multicolumn{3}{c}{Self-Convolution} \\
        & & Total (s) & BM (s)  & BMtime\% & Total (s) & BM (s) & Speed-Ups of BM \\
        \midrule
        \multirow{6}[0]{*}{D} & WNNM~\cite{gu2014weighted}  & 63.2 & 43.8 &36.9\% & 23.3 & 7.8 & \cellcolor{Gray}\textbf{3$\times$} \\
& STROLLR~\cite{wen2017strollr2d}& 87.7 & 68.9 & 36.7\% & 38.2 & 13.3 & \cellcolor{Gray}\textbf{3$\times$}  \\ 
& GHP~\cite{zuo2013texture} & 412.6 & 218.3 & 69.9\% & 288.6 & 94.2 & \cellcolor{Gray}\textbf{3$\times$}  \\ 
& NCSR~\cite{dong2012nonlocally} & 134.7 & 82.4 & 57.1\% & 76.9 & 28.1 & \cellcolor{Gray}\textbf{3$\times$}  \\ 
& SAIST~\cite{dong2012nonlocal} & 708.2 & 562.2 & 32.0\% & 227.0 & 78.6 & \cellcolor{Gray}\textbf{3$\times$}  \\ 
& PGPD~\cite{xu2015patch} &305.2 & 89.6 & 85.3\%& 260.3  & 41.3 & \cellcolor{Gray}\textbf{6$\times$} \\
        \midrule
\multirow{2}{*}{I} & TSLRA~\cite{guo2017patch} &1264.8 & 1063.4 & 24.0\%& 303.1  & 116.2 & \cellcolor{Gray}\textbf{3$\times$} \\
& GSR~\cite{zhang2014group} &2515.7 & 2314.9 & 16.3\%& 410.5  & 209.7 & \cellcolor{Gray}\textbf{2$\times$} \\
  \midrule
  \multirow{3}{*}{MD} & MCWNNM~\cite{xu2017multi} &2899.0 & 2371.3 & 15.8\%& 458.6  & 61.6 & \cellcolor{Gray}\textbf{8$\times$} \\ 
& SALT~\cite{wen2017joint}  & 375.9 & 113.8 & 75.4\% & 294.8  & 33.2 & \cellcolor{Gray}\textbf{9$\times$}  \\ 
& Self-MM & 139.0 & 44.3 & 78.8\% & 109.5 & 16.3 & \cellcolor{Gray}\textbf{7$\times$}\\
        \bottomrule
        \end{tabular}
    }
    \vspace{0.1in}
    \caption{Runtime (in seconds) comparisons among non-local algorithms using BM and Self-Convolution: Task D refers to denoising $512\times 512$ single-channel images; Task I refers to inpainting $256\times 256$ single-channel images; and Task MD refers to multi-channel denoising of $256 \times 256 \times q$ multi-channel images, where BMtime\% denotes the runtime portion of BM.}
  \label{tab:runtime results}
\end{table*}

We work on $512 \times 512$ gray images for denoising, $256 \times 256$ gray images for inpainting, and $256 \times 256 \times q$ multi-channel data (The data modalities are different, \ie color image, video, and multi-modality image with $q = 3$, $20$, and $6$ for MCWNNM, SALT, and Self-MM, respectively)~\footnote{We tested $20$ images of the same size and recorded the average runtimes, and also observed that the variance is very small.}, to test the original runtimes for image recovery by these methods.
In order to ensure a fair comparison, for those single-channel denoising and inpainting algorithms, we fix the patch stride to $1$ pixel, patch size as $6\times 6$, and search window size as $30\times 30$. For those multi-channel denoising algorithms, we fix the patch stride to $1$ pixel, patch size as $6 \times 6 \times d$ ($d = 3$, $1$, and $3$ for MCWNNM, SALT, and Self-MM, respectively), and search window size as $30\times 30 \times c$ (the third dimension of search windows are different, with $c = 1$, $9$, and $2$ for MCWNNM, SALT, and Self-MM, respectively, where $c=1$ refers to 2D patch searching and $c>1$ means 3D patch searching). Our Self-Convolution can also handle other settings, \eg using a larger stride or larger patch size, and we obtain similar or proportional speedups after testing (the related analysis can be seen in Sec.~\ref{sec:computation}).
Furthermore, we replace the conventional BM used in these methods with the proposed Self-Convolution, and record the accelerated runtime comparing to the original ones in Table~\ref{tab:runtime results}.
All competing algorithms generate the same result as with Self-Convolution, which proves the implementation equivalence in practice.
Besides, we calculate the runtime portion of BM in the original implementations of each algorithm denoted as BMtime\%.
For most of the non-local algorithms, BM occupies a major portion of the runtime, which limits the overall efficiency.
Table~\ref{tab:runtime results} lists the original and accelerated BM runtime of the selected algorithms.
We show that the proposed Self-Convolution can provide \textbf{2}$\times$ to \textbf{9}$\times$ speedup for BM, especially for multi-channel IR algorithms. 
All the experiments are carried out in the Matlab (R2019b) environment
running on a PC with Intel(R) Core(TM) i9-10920K CPU 3.50GHz.

\subsection{Multi-Modality Image Denoising}

\begin{figure*}[t]
	\begin{center}
		\begin{tabular}{c@{ }c@{ }c@{ }c@{ }c@{ }c}
			\includegraphics[width=.24\textwidth]{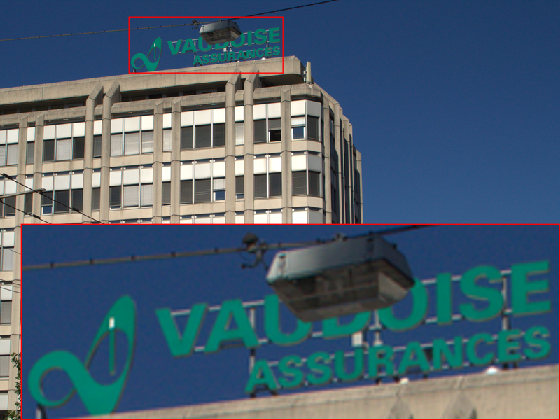}~&
			\includegraphics[width=.24\textwidth]{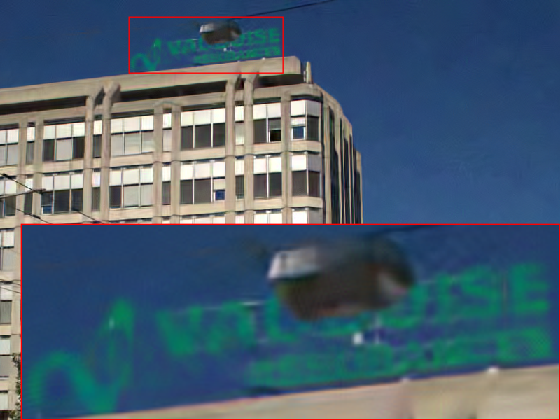}~&
			\includegraphics[width=.24\textwidth]{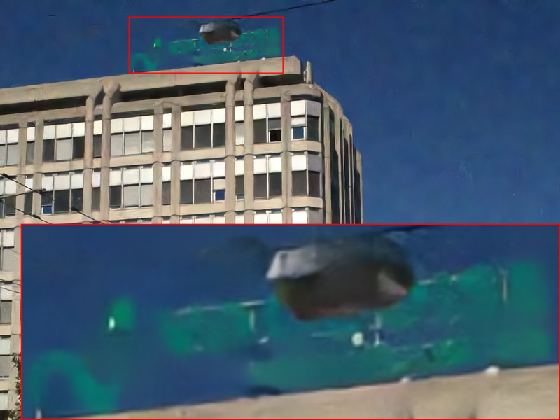}~&
			\includegraphics[width=.24\textwidth]{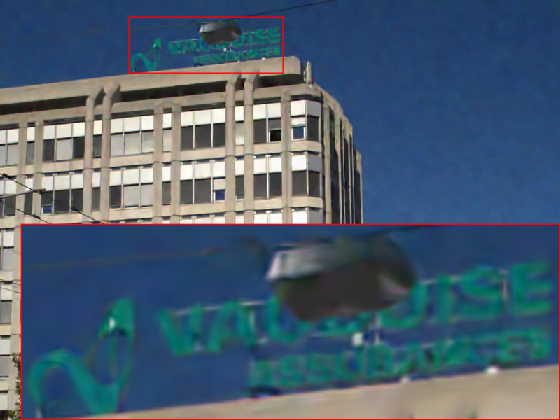}\\
			\includegraphics[width=.24\textwidth]{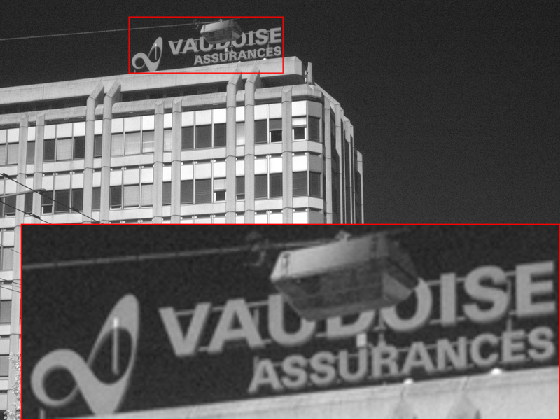}~&
			\includegraphics[width=.24\textwidth]{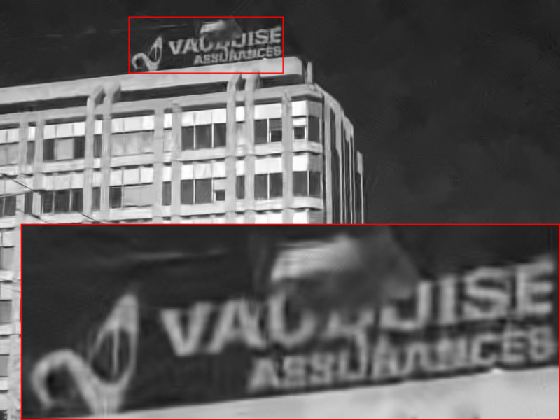}~&
			\includegraphics[width=.24\textwidth]{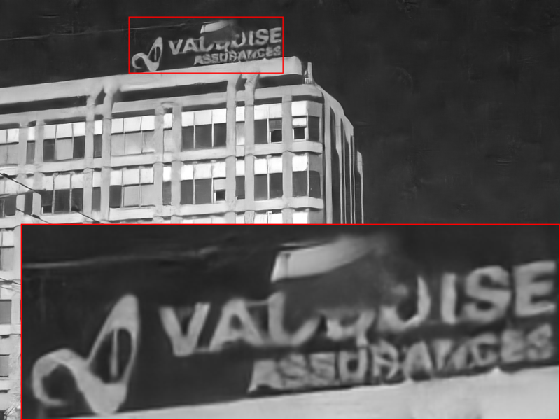}~&
			\includegraphics[width=.24\textwidth]{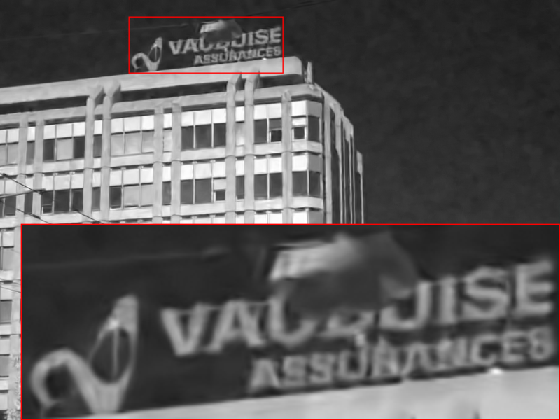}\\
	(a) Ground Truth & (b) C-BM3D, 26.78dB & (c) DnCNN, 27.05dB & (d) Ours, \textbf{27.22dB}\\
				\includegraphics[width=.24\textwidth]{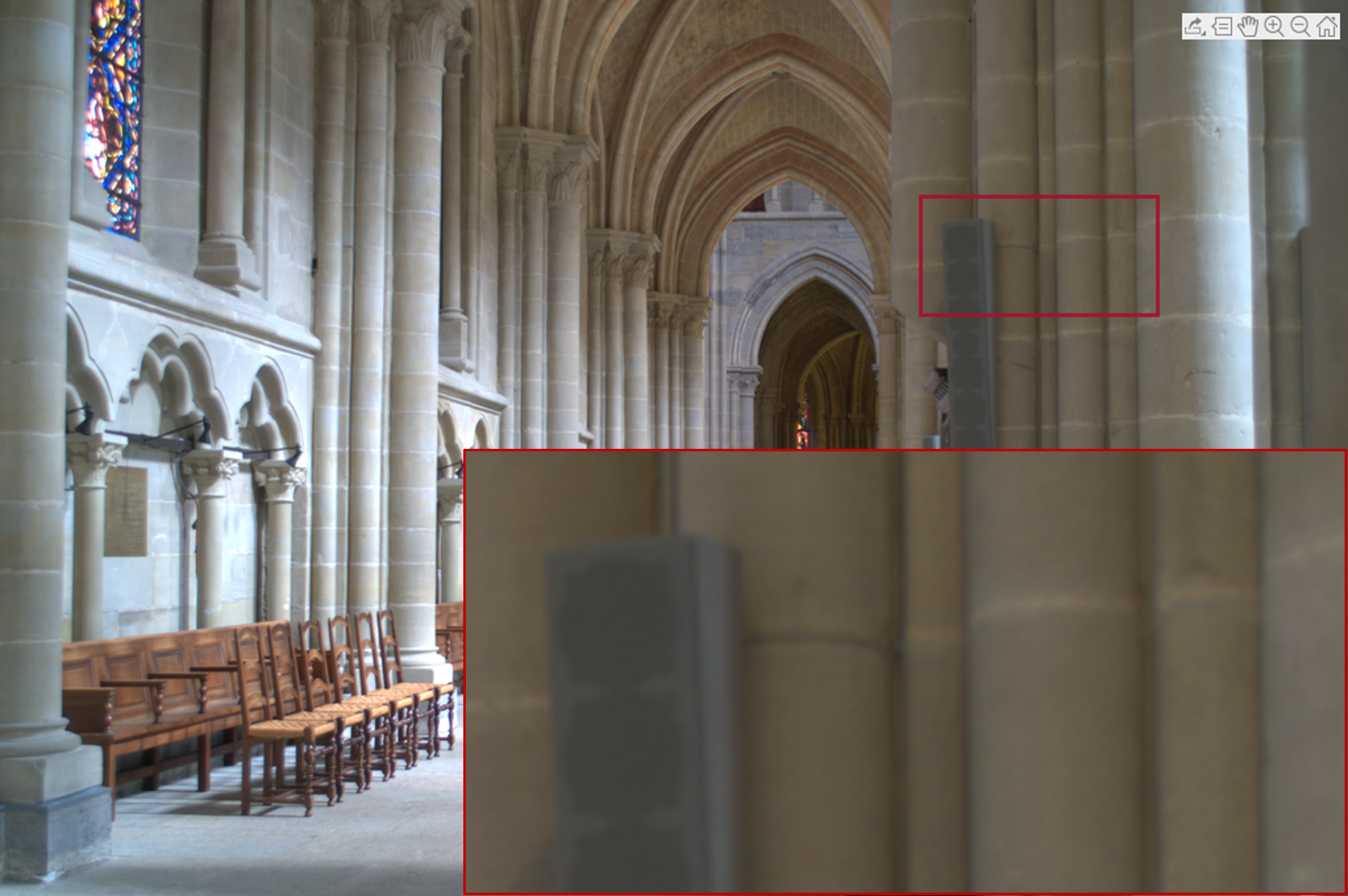}~&
				\includegraphics[width=.24\textwidth]{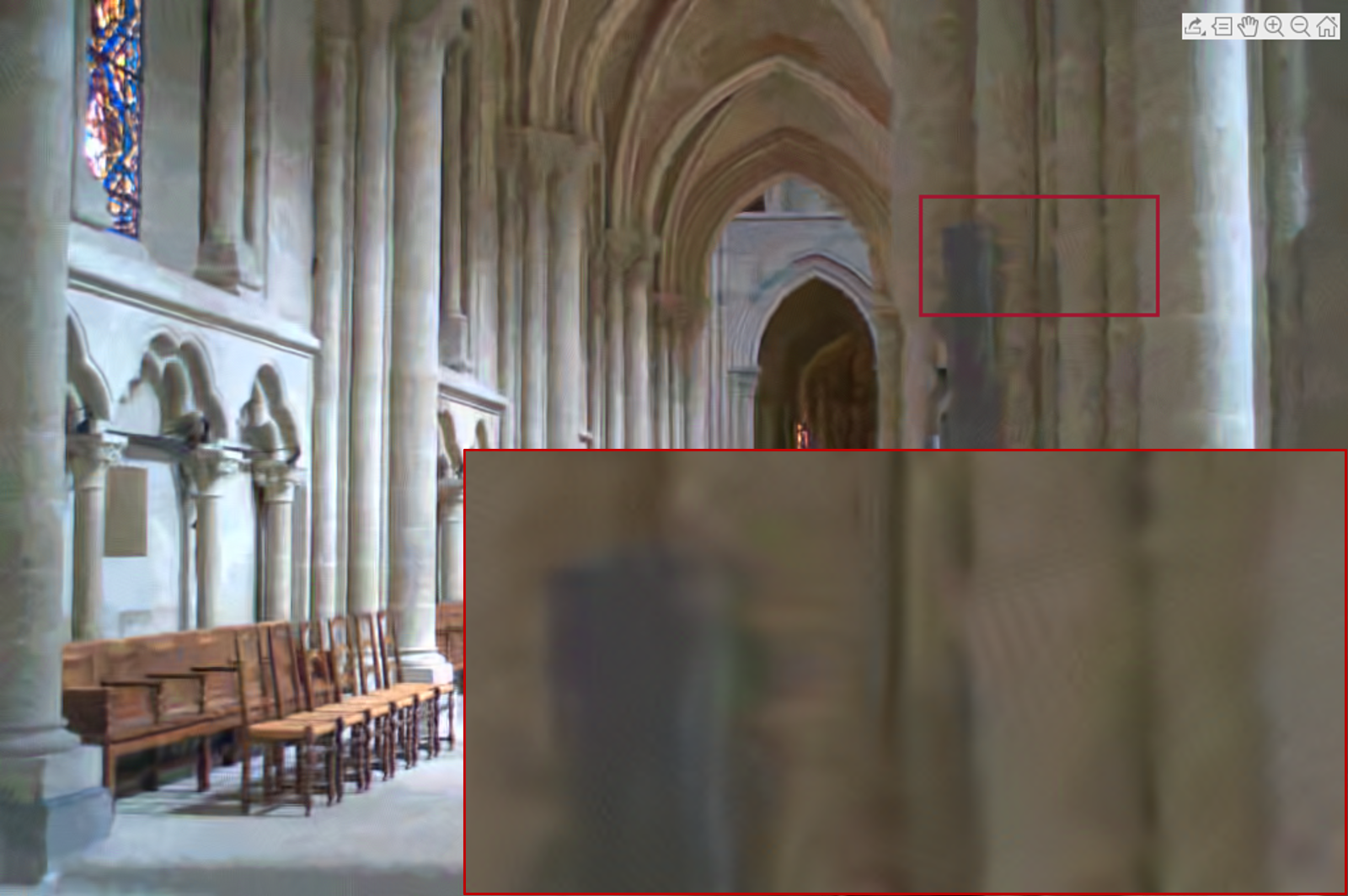}~&
				\includegraphics[width=.24\textwidth]{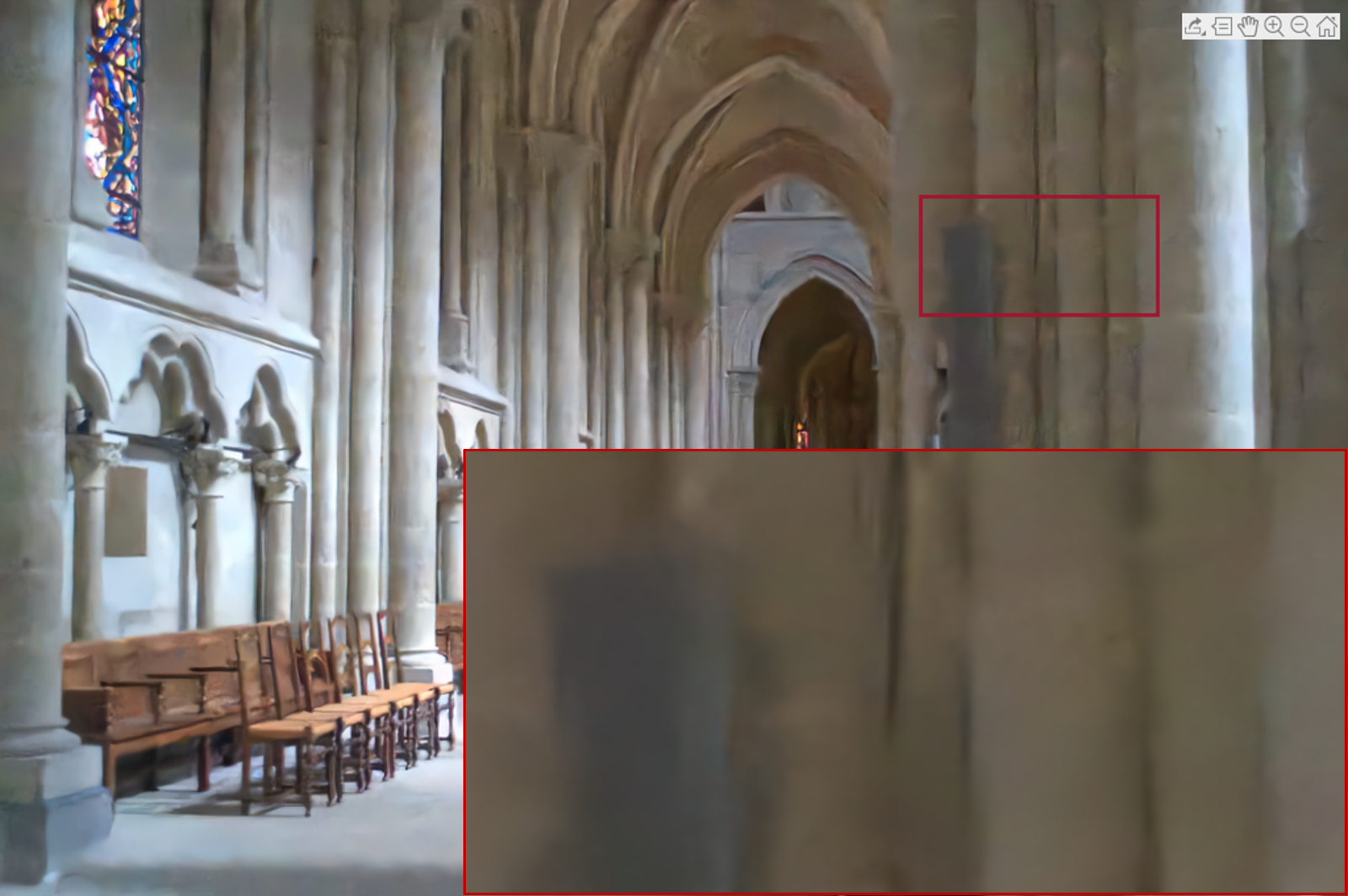}~&
				\includegraphics[width=.24\textwidth]{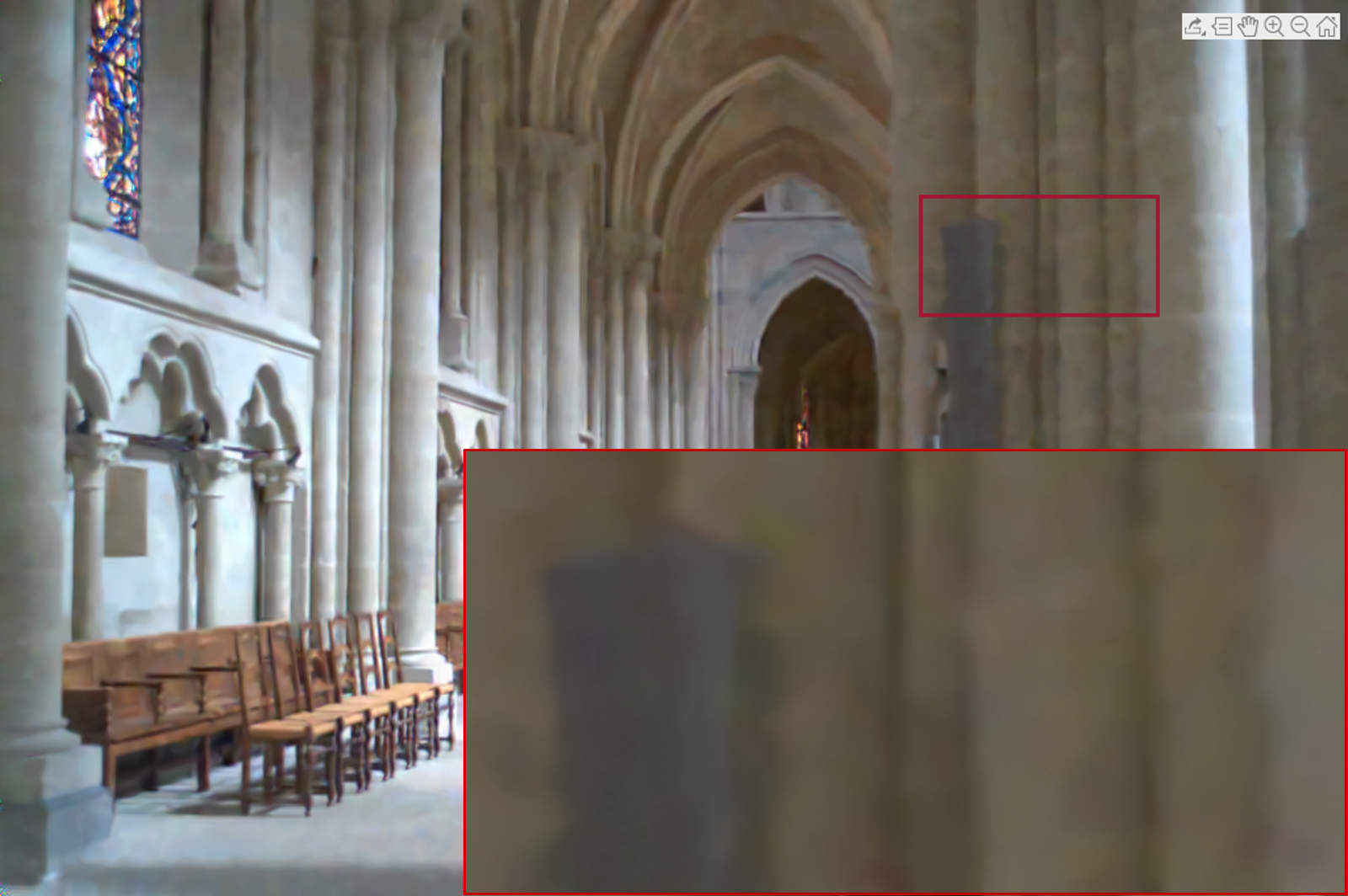}\\
				\includegraphics[width=.24\textwidth]{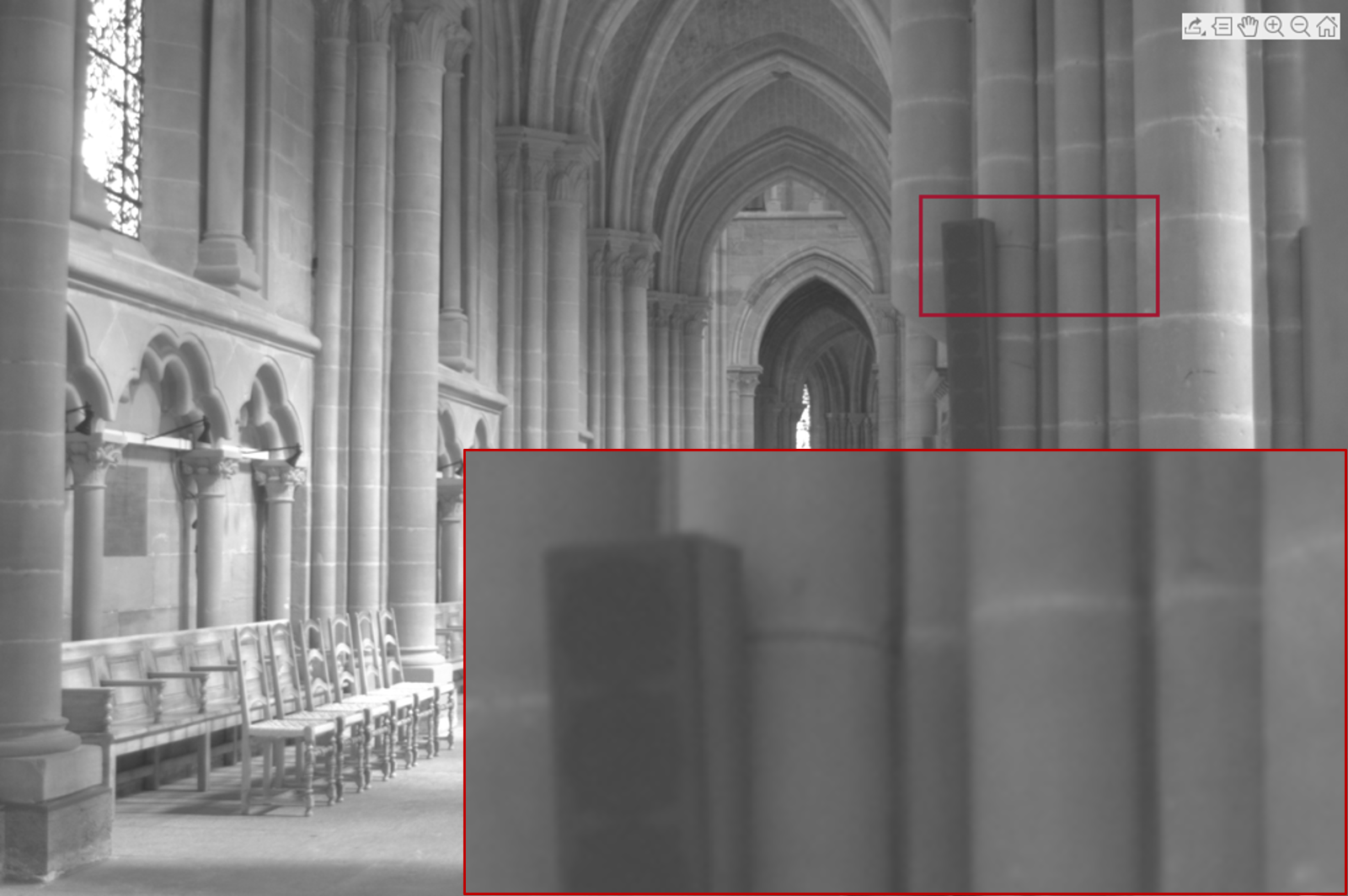}~&
				\includegraphics[width=.24\textwidth]{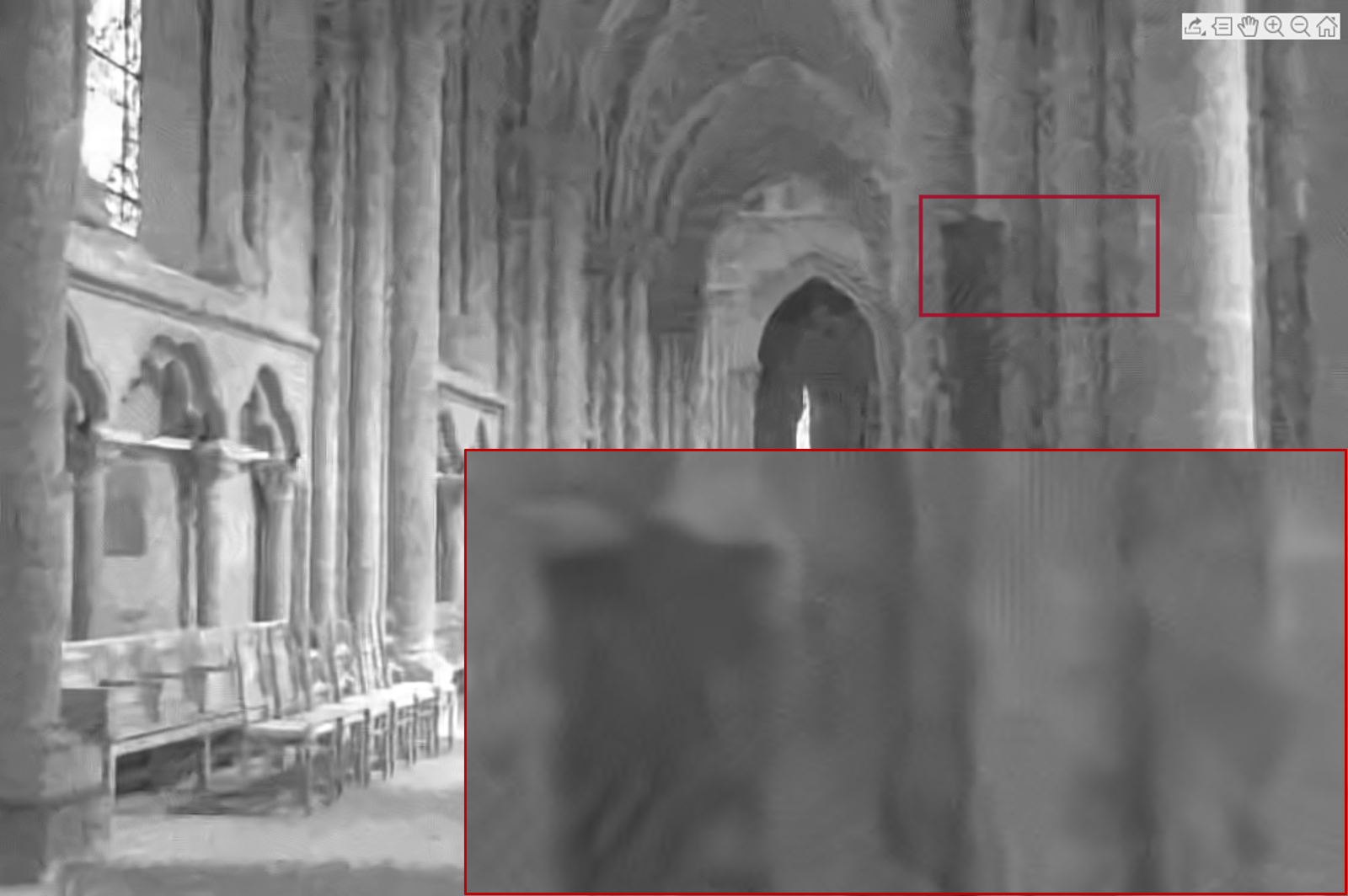}~&
				\includegraphics[width=.24\textwidth]{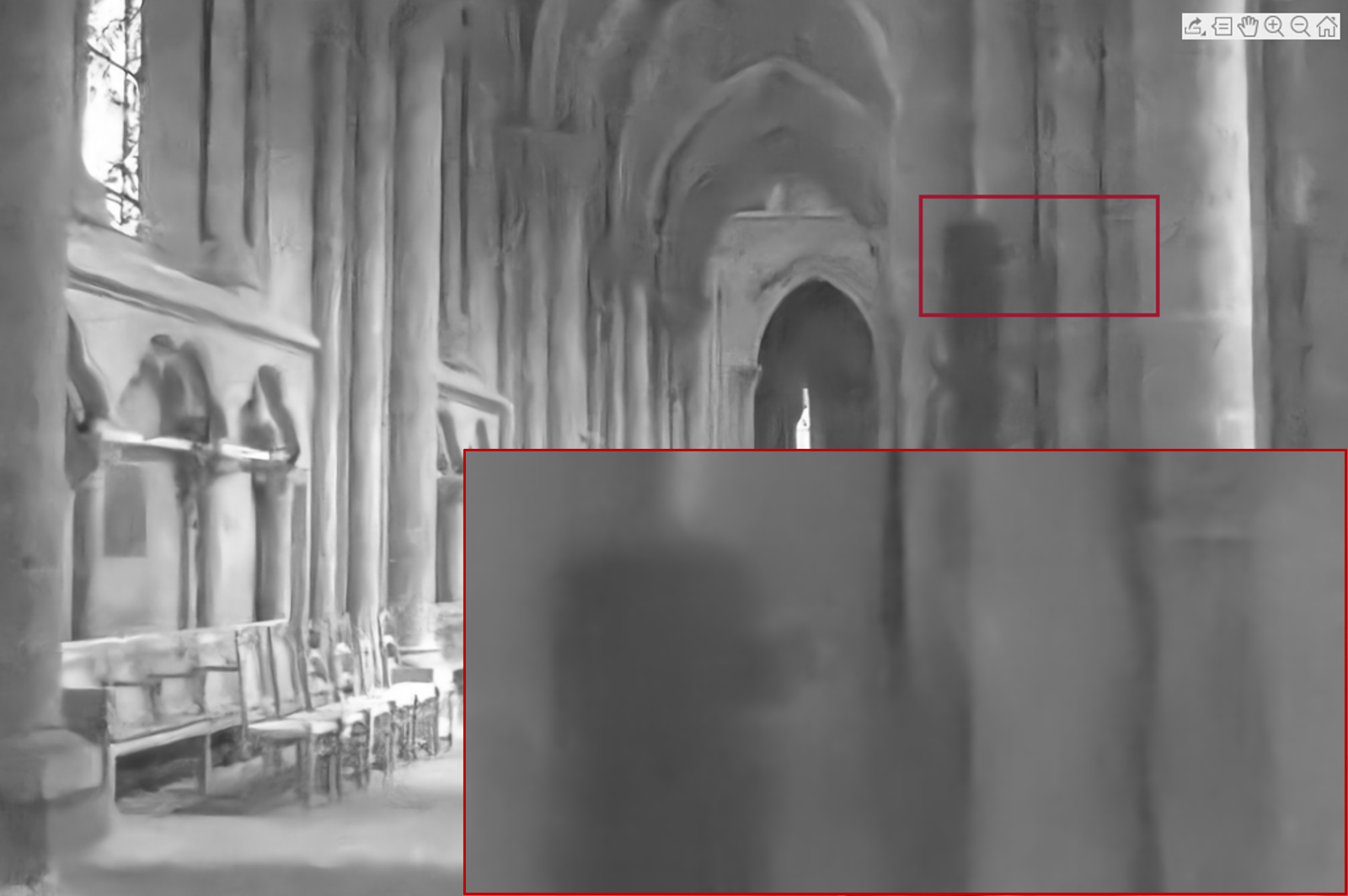}~&
				\includegraphics[width=.24\textwidth]{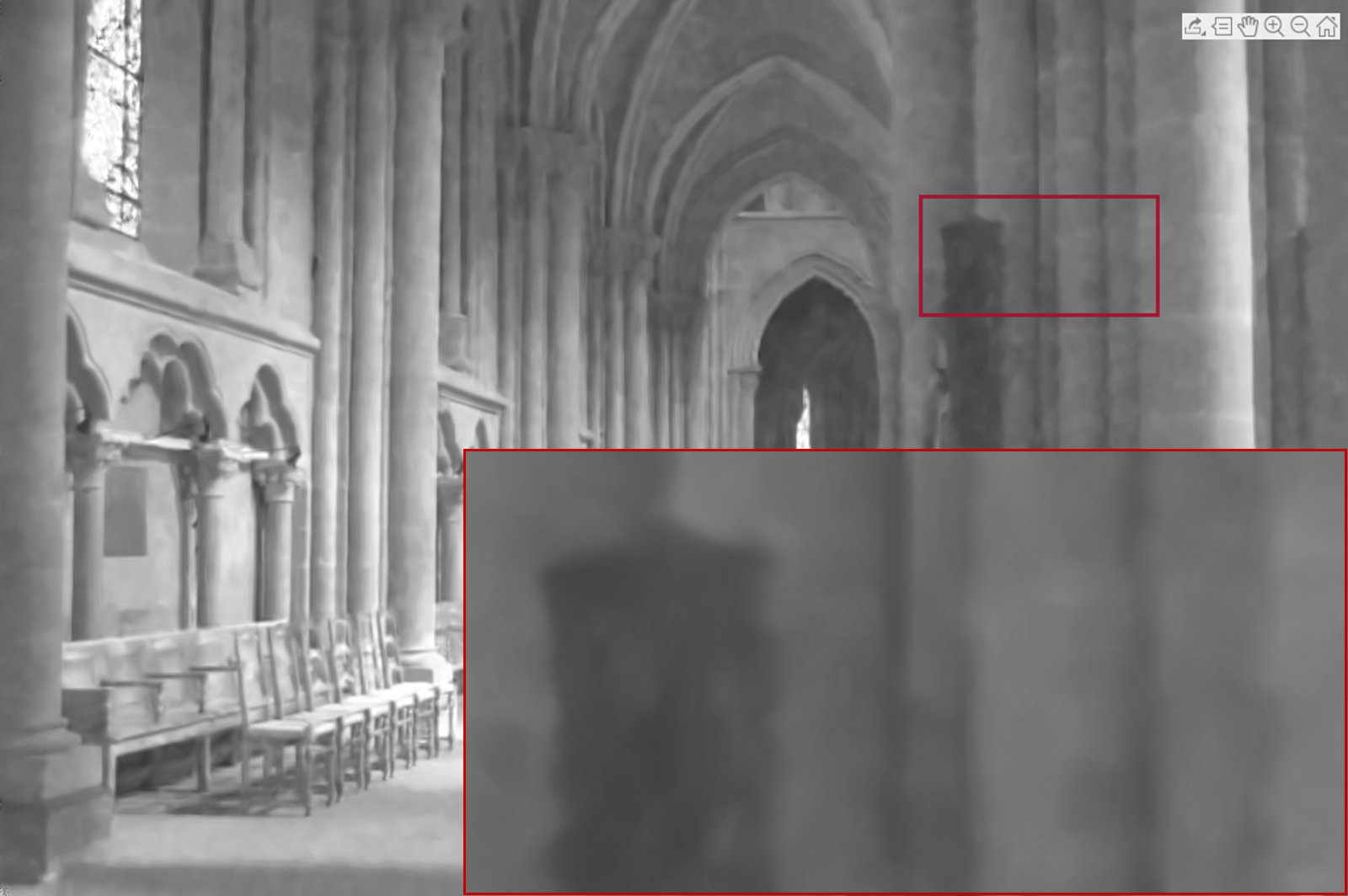}\\
	(a) Ground Truth & (b) C-BM3D, 31.07dB & (c) DnCNN, 30.78dB & (d) Ours, \textbf{31.60dB}\\
	\end{tabular}
	\end{center}
	
	\caption{Examples of the RGB-NIR image denoising ($\sigma=50$) results: top and bottom rows are RGB and NIR channels, respectively, of the (a) grouth truth, and the denoised results using (b) C-BM3D~\cite{dabov2006image,dabov2007color}, (c) DnCNN~\cite{zhang2017beyond}, and (f) our online Self-MM.}
	\label{fig:RGBNIR_res}
\end{figure*}

\begin{table}[!t]
\centering
\footnotesize
\setlength{\tabcolsep}{0.5em}
\adjustbox{width=1.\linewidth}{
    \begin{tabular}{l|ccc}
        \toprule
        Method &  Total Time (s) & BM Time (s) & PSNR (dB) \\ \midrule
        SAIST~\cite{dong2012nonlocal} & 9.5042e+03 &3.0413e+03 & 31.29 \\ 
        WNNM~\cite{gu2014weighted} &  2.9060e+03 & 970.5 & 31.81\\
        STROLLR~\cite{wen2017strollr2d} & 4.8876e+03 &1.7595e+03  & 31.62\\
        \rowcolor{Gray} Self-MM & \textbf{1.7914e+03} &\textbf{ 370.7} &\textbf{ 33.23}\\
        \bottomrule
    \end{tabular}
}
\vspace{0.1in}
   \caption{Runtime (in seconds) and PSNR performance (in dB) comparisons among existing non-local denoising algorithms using BM and our Self-MM averaged over $18$ images selected from the RGB-NIR Scene Dataset~\cite{BS11} ($\sigma=20$).}\label{tab:mm_runtime}
\end{table}

\begin{table}[!t]
\centering
\footnotesize
\setlength{\tabcolsep}{0.3em}
\adjustbox{width=1.\linewidth}{
    \begin{tabular}{l|ccccc}
        \toprule
         \multirow{2}{*}{Method} &      \multicolumn{5}{c}{Multi-Modality Image Denoising} \\
         &  $\sigma=5$ & $\sigma=10$& $\sigma=15$ & $\sigma=20$ & $\sigma=50$ \\ \midrule
       BM3D \cite{dabov2007image} & 39.04 & 35.22 & 33.15 & 31.75 & 27.66 \\ 
        C-BM3D~\cite{dabov2007image,dabov2007color} &  39.80 & 36.18 & 34.19 & 32.83 & 28.80 \\
        DnCNN \cite{zhang2017beyond} &  39.98 & \textbf{36.51} & 34.55 & 33.20 & 29.17 \\
        FOCNet \cite{jia2019focnet} & - & - & - & - & 28.50 \\
              Self-MM w/o TL & 40.01 & 36.26 & 34.34 & 32.91 & 28.96 \\
       Self-MM w/o LR & 39.53 & 35.90 & 33.86 & 32.44 & 28.62 \\
      \rowcolor{Gray}  Self-MM & \textbf{40.27} & \textbf{36.51} & \textbf{34.60} & \textbf{33.23} & \textbf{29.21}\\
      \midrule 
       \multirow{2}{*}{Method} &      \multicolumn{5}{c}{Guided Image Denoising} \\
         &  $\sigma=5$ & $\sigma=10$& $\sigma=15$ & $\sigma=20$ & $\sigma=50$ \\ \midrule
        Guided Filter~\cite{he2012guided} & 34.01 & 32.79 & 31.56 & 30.39 & 25.24 \\
        Cross Field~\cite{yan2013cross} & 36.70 & 33.31 & 31.67 & 30.09 & 23.97\\
      \rowcolor{Gray}  Self-MM (RGB) & \textbf{40.25} & \textbf{36.49} & \textbf{34.79} & \textbf{32.90} & \textbf{28.64} \\
        \bottomrule
    \end{tabular}
}
    \vspace{0.1in}
    \caption{PSNRs (in dB) for (1) RGB-NIR image denoising and (2) NIR guided RGB image denoising, averaged over $18$ images selected from the RGB-NIR Scene Dataset~\cite{BS11}. For those noise levels for which FOCNet's trained models are not publicly available, we put ``-'' as the corresponding results are unavailable.}
    	\label{tab:multi-modality results}
\end{table}

\subsubsection{Datasets}
We present experimental results to validate the promise of the proposed multi-modality image denoising scheme, and evaluate the effectiveness of proposed algorithm over the large-scale benchmark dataset, \ie IVRG RGB-NIR \cite{BS11}, which contains 447 pairs of near-infrared (NIR) and RGB images in 9 categories captured with separate exposures using modified SLR cameras. The scene categories include: country, field, forest, indoor, mountain, old building, street, urban, and water. We randomly selected 2 RGB-NIR pairs from each category, \ie 18 pairs in all, as testing images to evaluate the performance of our proposed algorithm. The selected testing images contain 4 channels (R, G, B, and NIR), with similar spatial resolutions of  approximately 1024$\times$768. We simulated i.i.d. additive Gaussian noise in all 4 channels with different $\sigma$ = 5,  10, 15, 20, and 50 to generate noisy measurements for each image.



\subsubsection{Implementation Details and Parameters}
The proposed Multi-Modality image denoising scheme uses a self-supervised approach, and there are several hyperparameters used in the algorithm, among which we directly use the noisy image as the initial estimate. We work with square patches of size $\sqrt{n} \times \sqrt{n} = 6\times6$ and set the spatial search window $\sqrt{N_s}\times \sqrt{N_s}=30 \times 30$. We initialize the unitary sparsifying transform $\mathbf{W}_0$ to be the 4D DCT (of size $\sqrt{n} \times \sqrt{n} \times K \times 
d$, where $K$ is the number of matched patches from BM and $d$ is the number of patch channels. We follow the guideline in~\cite{wen2020image}, and set the hyper-parameters in online Self-MM  to be $\gamma_s = \gamma_l = 1$. 
Similar to~\cite{wen2020image}, for low-level noise, \ie $0 \leq \sigma < 20$, we apply one-pass denoising scheme with $\beta = 3.3 \sigma$ and $\theta = 1.5 \sigma$; for high-level noise, \ie $\sigma \geq 20$, we apply multi-pass denoising scheme with $6$ iterations, \ie the output of the previous pass is weighted averaged with the original noisy image (weights of $0.9$ and $0.1$ for the previous output and the noisy image, respectively). 
Then, we re-estimate the variance of $\hat{\mathcal{Y}}_t$ based on $\sigma_{t}=\psi \sqrt{\sigma^{2}-(1 / N)\left\|\mathcal{Y}-\hat{\mathcal{Y}}_t\right\|^{2}}$, where the factor $\psi$ is set to be 0.71 to alleviate the over-estimated noise level. 
At the iteration $t=1,2,\dots,6$, we set the penalty parameter $\beta=0.9\sigma_{t-1}$ and $\theta = 0.8\sigma_{t-1}(\sqrt{n}+\sqrt{N})$~\cite{wen2017joint}, using the re-estimated $\sigma_{t-1}$ ($\sigma_0=\sigma$ at the first iteration).
Additionally, we set $K$ = 20, for $\sigma$ = 5, 10, 15, and $K=25$ for $\sigma = 20, 50$, respectively.

\subsubsection{Comparisons with State-of-the-Art Methods}
We compare our proposed algorithm to various popular or state-of-the-art denoising algorithms on efficiency and effectiveness.

\vspace{1mm}
\noindent\textbf{Efficiency evaluation.}
We demonstrate the efficiency comparison of the proposed online Self-MM and some popular classic non-local image denoising methods, \ie SAIST~\cite{dong2012nonlocal}, WNNM~\cite{gu2014weighted}, and STROLLR~\cite{wen2017strollr2d}~\footnote{Here we drop the BM3D~\cite{dabov2006image} algorithm since it was mainly implemented by C programming language, which is unfair to compare with other Matlab implemented algorithms.}.
We ran the methods on all testing RGB-NIR images and recorded the average run time, block matching time, and PSNR.
Existing algorithms are very time-consuming when handling such large-scale and high-dimensional data.
With the merits of the online learning scheme and the Self-Convolution algorithm, the proposed online Self-MM significantly outperforms the competing methods in both effectiveness and efficiency, as shown in Table~\ref{tab:mm_runtime}.

\vspace{1mm}
\noindent\textbf{Effectiveness evaluation.}
We first compare the proposed online Self-MM to existing popular image denoising algorithms, including the classic non-local methods BM3D~\cite{dabov2007image} and C-BM3D~\cite{dabov2007color}, and deep denoisers DnCNN~\cite{zhang2017beyond} and FOCNet~\cite{jia2019focnet}.
\add{
To denoise the four-channel data (RGB-NIR), for classic algorithms, we apply channel-wise BM3D denoising as well as BM3D $+$ C-BM3D (denoted as C-BM3D in Table~\ref{tab:multi-modality results} and Figure~\ref{fig:RGBNIR_res}) to denoise NIR and RGB components, respectively; For the deep denoisers, we apply the RGB and gray-scale DnCNN models to denoise the RGB and NIR components, respectively, and apply FOCNet for channel-wise denoising.}
Besides, to better understand the proposed online Self-MM scheme, we perform an ablation study by running our algorithm but only 
(i) applying low-rank penalties (denoted as ``w/o TL'') or only (ii) applying transform learning regularizer (denoted as ``w/o LR'').
We also compare our method to guided denoising methods, \ie Guided Filtering~\cite{he2012guided} and Cross-Field Filtering~\cite{yan2013cross}.
As the guided denoising methods~\cite{he2012guided,yan2013cross} are for general image fusion (\ie not mainly for denoising), and require the guidance channel to be clean, we follow the same setup with the oracle clean NIR channels available as guidance for RGB-NIR data and report PSNR values for RGB denoising.
To compare with these guided denoising methods, we also vary the proposed online Self-MM to denoise only the RGB channels, denoted as ``Self-MM (RGB)''.
Table~\ref{tab:multi-modality results} lists the PSNRs (in dB) of RGB-NIR denoising results for different noise levels. The displayed PSNRs are averaged over the randomly selected $18$ images from the IVRG RGB-NIR Dataset~\cite{BS11}.
The proposed algorithm outperforms all competing methods on average, and for most noise levels. 
Besides, as a self-supervised learning algorithm, our method generates superior results when the noise level $\sigma$ is low. 
On the contrary, deep algorithms demonstrate advantages when the data become more noisy.
Fig.~\ref{fig:RGBNIR_res} shows examples of the denoised RGB-NIR images ($\sigma =50$), where our method preserves more detailed textures and alleviates distortion with better visual quality, compared to competing methods.

\section{Conclusion}
In this paper, we proposed a novel Self-Convolution operator to unify non-local image modeling. 
We showed that Self-Convolution generalizes many commonly used non-local operators, including block matching and non-local means. 
Based on Self-Convolution, we proposed a multi-modality image restoration scheme called online Self-MM. 
Extensive experiments show that (1) Self-Convolution can speed up block matching up to $9\times$ in many popular algorithms and (2) Online Self-MM outperforms many popular algorithms for denoising RGB-NIR images in terms of both effectiveness and efficiency.
\add{We plan to study more structured Self-Convolution with further acceleration in future work.}


%

\appendices
\setcounter{lemma}{0}
\setcounter{proposition}{0}
\setcounter{theorem}{0}



\ifCLASSOPTIONcaptionsoff
  \newpage
\fi



\bibliographystyle{IEEEtran}
\bibliography{myref}
\end{document}